\pdfoutput=1
\documentclass[10pt,twocolumn,letterpaper]{article}
\usepackage[table]{xcolor}
\usepackage[pagenumbers]{cvpr}
\usepackage[T1]{fontenc}
\usepackage[utf8]{inputenc}
\usepackage{amsmath,amssymb,amsthm}
\usepackage{graphicx}
\usepackage{stfloats}
\usepackage{booktabs}
\usepackage{multirow}
\usepackage{algorithm}
\usepackage[noend]{algpseudocode}
\usepackage{microtype}
\definecolor{cvprblue}{rgb}{0.21,0.49,0.74}
\usepackage[pagebackref,breaklinks,colorlinks,allcolors=cvprblue]{hyperref}
\hypersetup{pdftitle={Speculative Sampling in Interactive Video Generation},pdfauthor={Yubo Huang, Yuxuan Zhang, Weiqiang Wang, Jiaming Liu, Sirui Zhao, Enhong Chen}}
\def\confName{CVPR}
\def\confYear{2027}
\title{Speculative Sampling in Interactive Video Generation}
\author{%
  Yubo Huang\raisebox{0.55ex}{\scriptsize 1}\quad
  Yuxuan Zhang\raisebox{0.55ex}{\scriptsize 2}\quad
  Weiqiang Wang\raisebox{0.55ex}{\scriptsize 3}\quad
  Jiaming Liu\raisebox{0.55ex}{\scriptsize 4}\quad
  Sirui Zhao\raisebox{0.55ex}{\scriptsize 1}\quad
  Enhong Chen\raisebox{0.55ex}{\scriptsize 1}\thanks{\href{mailto:cheneh@ustc.edu.cn}{cheneh@ustc.edu.cn}}\\[6pt]
  \normalsize\raisebox{0.45ex}{\scriptsize 1}University of Science and Technology of China\\
  \normalsize\raisebox{0.45ex}{\scriptsize 2}The Chinese University of Hong Kong\\
  \normalsize\raisebox{0.45ex}{\scriptsize 3}Monash University\\
  \normalsize\raisebox{0.45ex}{\scriptsize 4}Independent Researcher
}

\algrenewcommand\algorithmiccomment[1]{\hfill{\color{gray!70!black}$\triangleright$~#1}}

\newcommand{\method}{Speculative Sampling}
\newcommand{\E}{\mathbb{E}}

\definecolor{tableheader}{gray}{0.93}
\definecolor{tablerule}{gray}{0.28}
\definecolor{tablestepone}{RGB}{241,247,252}
\definecolor{tablesteptwo}{RGB}{247,247,247}
\definecolor{tablestepfour}{RGB}{252,247,239}

\newcommand{\committablestyle}{%
  \footnotesize
  \renewcommand{\arraystretch}{0.96}%
  \setlength{\tabcolsep}{2.1pt}%
  \setlength{\aboverulesep}{0.28ex}%
  \setlength{\belowrulesep}{0.28ex}%
  \arrayrulecolor{tablerule}%
}
\newtheorem{proposition}{Proposition}

\begin{document}
\maketitle
\begin{abstract}
Interactive video generation commonly relies on small bidirectional chunks to
reduce response latency. Yet our preliminary experiments show that larger
chunks improve video quality and better utilize GPU parallelism, allowing more
latents to be sampled with only a modest increase in sampling time. Motivated
by these observations, we propose Speculative Sampling, a strategy that
combines the modeling power of large bidirectional windows with the low-latency
interaction of small ones. It expands the modeling window by appending noise
latents as padding and caches the resulting predictions for reuse when the
condition remains unchanged. Predictions are delivered incrementally, while
the uncommitted suffix remains replaceable. Our latency analysis shows that
this removes the waiting-time lower bound imposed by chunk boundaries, shifting
the optimization focus from reducing chunk size to accelerating per-chunk
generation. Bidirectional window size thus becomes a decision variable for
jointly optimizing visual quality and generation time, rather than a necessary
constraint on low-latency interaction. This opens a path to serving large
bidirectional models in interactive streaming systems without complex
streaming post-training.
Project page: \url{https://yubo-shankui.github.io/speculative_sampling/}.
\end{abstract}

\section{Introduction}
\label{sec:intro}

Interactive video generation requires newly arriving conditions to influence
ongoing playback promptly. A viewer may change a camera command or action
instruction at any time, while existing video continues to play.
Generation throughput alone does not capture this requirement: a fast model
can still leave the viewer waiting for a generation or playback boundary.

A natural route is fine-grained autoregressive generation. Advances in
rollout training and few-step distillation have made this route increasingly
practical \citep{huang2025self,zhu2026causalpp,yang2025longlive}.
Smaller generation chunks provide more frequent opportunities to incorporate
new conditions, enabling fine-grained interaction, but may limit joint
temporal modeling and parallel efficiency.
This raises our central question: \emph{must low-latency interaction come at
the cost of a smaller bidirectional modeling window?}

We introduce Speculative Sampling, a strategy that combines the modeling power
of large bidirectional windows with the low-latency interaction of small ones.
It expands the modeling window by appending noise latents as padding and caches their predictions
for reuse when the condition remains unchanged, while allowing the
uncommitted suffix to be revised when new conditions arrive (Figure~\ref{fig:concept}).

\begin{figure}[t]
  \centering
  \centering
  \includegraphics[width=\linewidth]{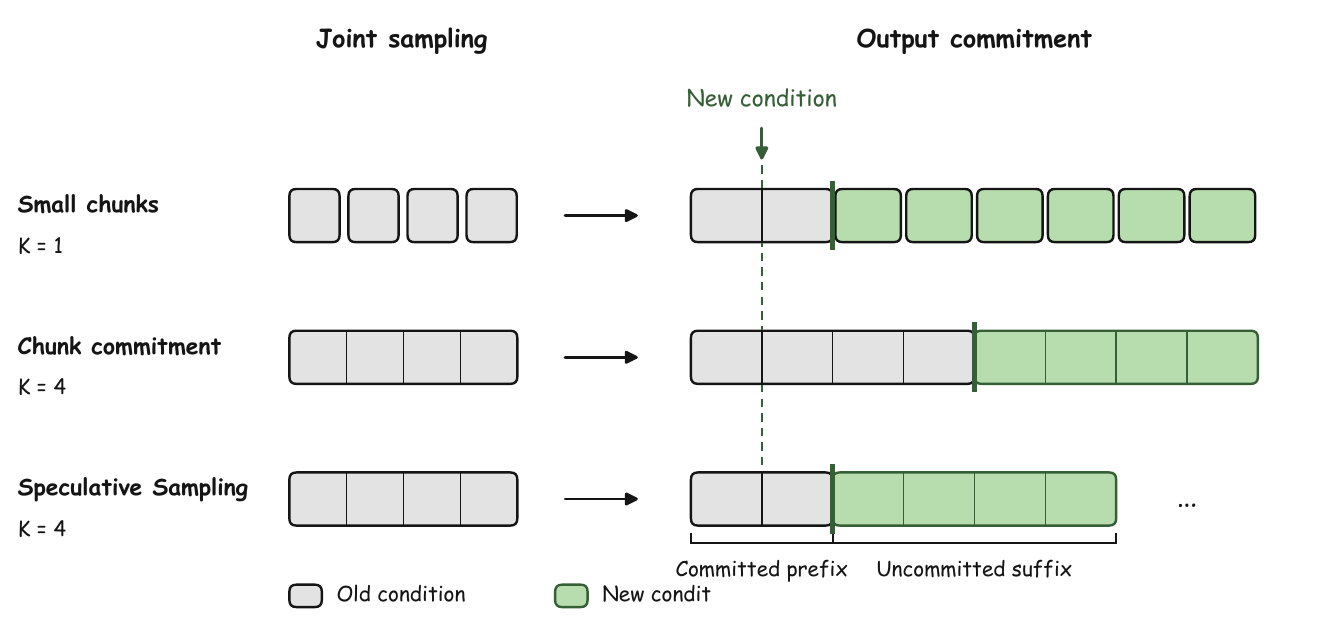}
  \caption{Fine-grained interaction without shrinking generation chunks.
  Our method preserves large-chunk joint sampling while allowing updates
  beyond the committed prefix, rather than only at chunk boundaries.
  Gray/green denote old/new conditions; the dashed line marks playback
  position at condition arrival, and thick green edges mark updated output.
  Positions represent video content, not computation time.}
  \label{fig:concept}
\end{figure}

We first formulate a serving model that separates condition arrivals,
generation, output commitment, and playback. Generation is the computation
needed for deliverable video output; for the latent diffusion models studied
here, it includes DiT sampling and required VAE decoding. The model distinguishes
generation time from phase-dependent \emph{random-arrival waiting} imposed
by scheduling and playback boundaries, and identifies replaceable outputs.
When new output can enter playback only at chunk boundaries, the chunk
duration imposes a waiting floor even if generation is fast.

We next examine chunk size $K$, the number of latent frames jointly
sampled in one DiT sampling call. Changing $K$ affects joint temporal modeling
and parallel efficiency, as well as the time needed to generate each chunk.
Our chunk-size ablation shows higher aggregate video
quality for larger-chunk configurations than their frame-wise counterparts
in the evaluated settings (Section~\ref{sec:model-profiling}). Larger chunks also
better utilize GPU parallelism, allowing more latents to be sampled with only
a modest increase in sampling time. Together with the
system model, this motivates retaining the benefits of joint generation
while reducing the waiting imposed by coarse output commitment.

Guided by the system model and chunk-size ablation, we develop our method to
combine large-chunk joint generation with fine-grained delivery.
We expand the bidirectional modeling window by appending noise latents as
padding. Rather than discarding the resulting predictions, we cache them for
reuse when the condition remains unchanged, while keeping the uncommitted
suffix replaceable.
We formulate the controller as a two-state finite-state machine that
coordinates joint generation and incremental consumption of cached
predictions, while playback continues independently.
This decouples DiT sampling calls from individual deliveries: completed
predictions are reused between condition updates, avoiding repeated
sampling merely to support finer-grained delivery.

The resulting design points to an alternative route to interactive video
generation: shrinking chunks need not be a prerequisite for low interaction
latency. Instead, the optimization focus can shift to accelerating
high-quality joint generation.
DiT sampling acceleration and parallel acceleration can support this route.
We further demonstrate this route by serving a large bidirectional model in a
low-latency interactive streaming system without Self-Forcing-style
post-training.

Our contributions are fourfold:
\begin{enumerate}
  \item \textbf{A simple serving-system model for interactive video generation.}
  We formalize condition arrivals, generation, output commitment, and
  playback, distinguishing generation time from the additional waiting
  imposed by scheduling and playback boundaries
  (Section~\ref{sec:problem}).
  \item \textbf{A systematic analysis of chunk size.}
  We study how chunk size affects video quality, generation
  time, and redundant computation, clarifying the trade-offs between joint
  modeling and responsiveness (Section~\ref{sec:model-profiling};
  Appendices~\ref{app:latency}--\ref{app:cost}).
  \item \textbf{Recomputation-efficient speculative sampling.}
  Guided by this model and analysis, we propose a speculative sampling
  scheme formulated as a two-state finite-state machine.
  It combines large-chunk joint generation with fine-grained output
  commitment, reusing predictions and replanning only
  the uncommitted suffix when conditions change (Section~\ref{sec:method}).
  \item \textbf{An alternative route to interactive video generation.}
  By separating chunk size from output commitment, we shift
  the optimization focus from shrinking chunks to accelerating high-quality
  joint generation, making per-chunk generation time a central target
  for improving responsiveness.
\end{enumerate}

\section{Related work}
\label{sec:related}

\paragraph{Autoregressive video generation.}
Autoregressive video generation exploits temporal structure by predicting new
content conditioned on the past \citep{yan2021videogpt}. This provides a
natural interface for extending videos, including segment-wise continuation
under changing prompts \citep{villegas2022phenaki}. Bringing sequential
prediction to diffusion models \citep{chen2024diffusionforcing} leaves two
practical challenges: the cost of repeated sampling and errors accumulated
through generated history. Recent progress therefore targets efficient
few-step causal generation \citep{yin2024causvid,zhu2026causalpp} and training
aligned with self-generated, long-horizon rollouts
\citep{huang2025self,yang2025longlive}. Rolling-window methods, including
Rolling Forcing, combine joint temporal modeling with incremental output
through progressive denoising
\citep{kim2024fifo,kodaira2025streamdit,liu2025rolling}.
Related bidirectional-window approaches also support sequential generation
\citep{biwm2026}.
NFD additionally describes speculative multi-frame sampling on a frame-causal
backbone to improve throughput, but retains a causal mask within the enlarged
window and therefore does not exploit bidirectional modeling across it
\citep{cheng2025nfd}. Related work also studies generation--playback
lead management \citep{catnip2026mainecoon}. We instead retain
chunk-wise joint sampling and manage the reuse and replacement of completed
predictions during interactive serving.

\paragraph{Accelerating video generation.}
Acceleration targets several components of generation cost. Few-step
distillation \citep{yin2024causvid,zhu2026causalpp} and adversarial post-training
\citep{lin2025adversarial} reduce DiT sampling steps. Pyramid Flow performs
much of its DiT sampling below full spatial resolution and compresses history
through a temporal pyramid \citep{jin2024pyramid}; FramePack reduces the
historical-context token budget \citep{zhang2025framepack}.
FasterCache and TeaCache reuse intermediate computations across diffusion
steps \citep{fastercache2024,teacache2024}, whereas Sliding Tile Attention
and Sparse VideoGen reduce attention work through local or spatial--temporal
sparsity \citep{sta2025,sparsevideogen2025}. USP addresses distributed attention
through sequence parallelism \citep{fang2024usp}. These methods have distinct
training and inference requirements. Our perspective shifts the emphasis
from shrinking generation chunks to accelerating their generation. By
decoupling generation from output commitment, our method highlights per-chunk
generation time, rather than chunk size alone, as a key determinant of
interaction latency. This places greater importance on accelerating joint
generation instead of requiring increasingly fine-grained autoregressive
models.

\section{Method}
\label{sec:method}

\subsection{Serving-System Model}
\label{sec:problem}

An interactive video serving system generates video into a buffer while
the player consumes previously generated content. New conditions arrive
online and must influence subsequent playback promptly.
For the latent diffusion models considered here, the \emph{chunk size}
$K$ counts the latents jointly sampled with \textbf{bidirectional}
temporal modeling.\footnote{Throughout this paper, our chunk-based
generation setting uses bidirectional temporal modeling within each
chunk and autoregressive continuation across chunks, conditioned on
retained past content.}
Let $f$ denote the video's temporal sampling frame rate, used for generation
and playback at its intended speed, and $r$ the pixel frames represented
by each latent. A latent spans $\delta=r/f$ seconds, so a chunk covers
$T=K\delta=Kr/f$ seconds.

\begin{figure}[htbp]
  \centering
  \begin{minipage}{0.9\linewidth}
  \centering
  \includegraphics[width=\linewidth]{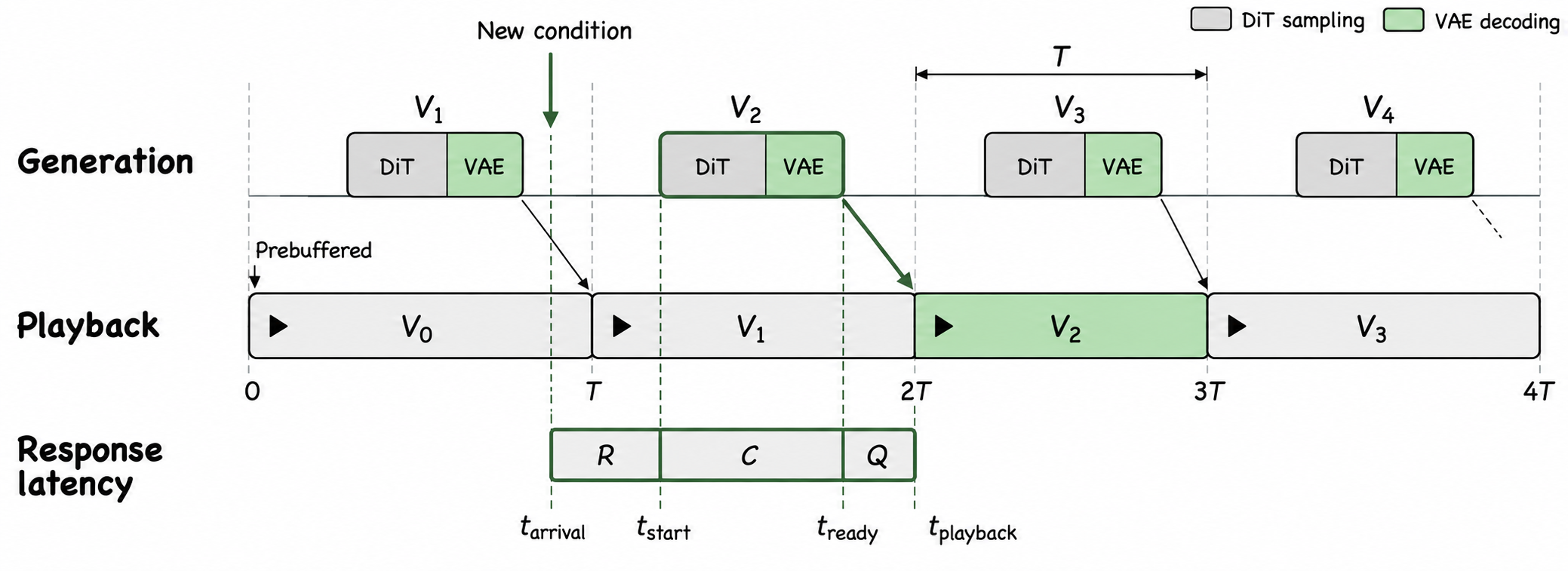}
  \end{minipage}
  \caption{Generation speed alone does not determine response latency.
  In this periodic reference schedule, the highlighted update waits $R$
  for generation, takes $C$ to produce deliverable output, and waits $Q$
  for playback. Each chunk is prepared for the next playback period;
  $V_0$ is prebuffered.}
  \label{fig:serving-queues}
\end{figure}

Figure~\ref{fig:serving-queues} illustrates why response latency is not
just DiT sampling time. A condition arrives at $t_{\mathrm{arrival}}$,
waits until generation starts at $t_{\mathrm{start}}$, and yields the first
output frames ready at $t_{\mathrm{ready}}$ and displayed at
$t_{\mathrm{playback}}$.
For a displayed update, these stages give
\begin{equation}
 \begin{aligned}
 R&=t_{\mathrm{start}}-t_{\mathrm{arrival}},&
 C&=t_{\mathrm{ready}}-t_{\mathrm{start}},\\
 Q&=t_{\mathrm{playback}}-t_{\mathrm{ready}},&
 L&=R+C+Q.
 \end{aligned}
 \label{eq:latency}
\end{equation}
Here $R$ and $Q$ are the waits before generation and playback, respectively;
together they account for the random-arrival waiting in this model.
Generation time $C$ covers the critical path to the first deliverable
output, including the required DiT sampling and VAE decoding.

For a conventional interface that admits newly conditioned output only
at fixed chunk playback boundaries, the expected combined waiting satisfies
$\E[R+Q]\geq T/2$ under fixed generation time and isolated, uniformly
phased arrivals (Section~\ref{sec:analysis}).
Since $T=Kr/f$, two direct ways to reduce this waiting floor at fixed $r$
are \textbf{smaller chunks} and \textbf{higher temporal frame rates}.
Both approaches reduce latency by shortening the chunk period $T$.
We instead ask: \textbf{can we eliminate the chunk-period waiting floor
by decoupling joint generation from output commitment?}
This would remove the need to shrink chunks or increase the model's
temporal frame rate merely to obtain finer-grained interaction, shifting
the optimization focus to generation time rather than chunk duration.

\begin{algorithm*}[t]
\caption{\method{} controller (delivery granularity $q=1$).}
\label{alg:controller}
\small
\begin{algorithmic}[1]
\Require generator $p_\theta$, chunk size $K$, latent duration $\delta$,
  time reserved for generation $\eta$, initial condition $c_0$
\Statex \textbf{State:} prediction cache $B$; conditioning state
  $\mathcal H$; latest and pending conditions $v,u$; delivery cursor $p$;
  committed-prefix end $\ell$; prediction end $n$; active job $J$ at
  boundary $j_J$
\State $v\gets c_0$;\; $u\gets\varnothing$;\; $p\gets0$;\;
  \Call{Launch}{$0,c_0$} \Comment{bootstrap}
\For{each control-clock tick}
  \For{each event recorded since the previous tick, in arrival order}
  \State \textbf{on} $\mathsf{update}(c)$:\; $v\gets c$;\; $u\gets c$
    \Comment{active $J$ keeps its condition}
  \State \textbf{on} $\mathsf{ready}(m)$ from $J$:\;
    mark $B[j_J{+}m]$ deliverable
    \Comment{may precede $\mathsf{finish}$}
  \State \textbf{on} $\mathsf{finish}$ from $J$:\; validate its outputs;\;
    $n\gets j_J{+}K$;\; $J\gets\varnothing$
  \State \textbf{on} $\mathsf{delivery}$ (every $\delta$):\;
    \textbf{if} $B[p]$ deliverable \textbf{then} deliver it,
    $p\gets p{+}1$, $\ell\gets\max(\ell,p)$ \textbf{else} log underrun
  \EndFor
  \State \textbf{if} $J\neq\varnothing$ \textbf{then continue}
    \Comment{$\mathsf{Produce}$}
  \State $\mathcal F\gets\{\,j : \ell\leq j\leq n,\ \textsc{Feasible}(j)\,\}$
    \Comment{Eq.~\eqref{eq:fsm-feasible}; $j=n$ appends}
  \State \textbf{if} $u\neq\varnothing$ \textbf{and}
    $\mathcal F\neq\varnothing$ \textbf{then}
    \Call{Launch}{$\min\mathcal F,\,u$};\; $u\gets\varnothing$
    \Comment{replacement}
  \State \textbf{else if} $u=\varnothing$ \textbf{and}
    $(n-p)\,\delta\leq\eta$ \textbf{and} $n\in\mathcal F$ \textbf{then}
    \Call{Launch}{$n,\,v$}
    \Comment{refill}
  \State \textbf{else} no launch
    \Comment{$\mathsf{Consume}$; a pending $u$ is kept}
\EndFor
\Procedure{Launch}{$j,c$}
  \Comment{the only step that starts DiT sampling}
  \State $\ell\gets j$;\; $\mathcal H\gets\Call{Retain}{\mathcal H,B[{:}j]}$
    \Comment{commit prefix}
  \State invalidate $B[j{:}]$ and dependent state;\;
    $n\gets j$;\; $J\gets(j,c)$
  \State asynchronously sample
    $\widetilde Y\sim p_\theta(\cdot\mid\mathcal H,c)$, $|\widetilde Y|=K$,
    into $B[j{:}j{+}K]$
    \Comment{Eq.~\eqref{eq:sample}}
  \State VAE-decode the chunk, emitting $\mathsf{ready}(m)$ when the frames
    for latent $m$ are ready, then $\mathsf{finish}$
\EndProcedure
\end{algorithmic}
\end{algorithm*}

\subsection{Speculative Sampling via a Finite-State Machine}
\label{sec:fsm}

Our design separates how much the model jointly predicts from how much
the system delivers at a time. \emph{Delivery granularity} $q$ is
the number of latents whose pixel frames are transferred together;
our method uses $q=1$, corresponding to $\delta$ seconds of video.
For each DiT sampling call, we append $K-1$ noise latents as padding to the
next latent to be delivered, forming a $K$-latent bidirectional modeling
window. Rather than discarding the resulting padding predictions, we retain
them in a prediction cache for reuse while the condition remains unchanged.
\emph{Output commitment} fixes latents so they can no longer be modified;
they need not have been delivered or played yet. We call these latents the
\emph{committed prefix}. The remaining predictions form the
\emph{uncommitted suffix}, which can be reused or replaced.
This is the role of \emph{speculative sampling}, distinct from the
individual DiT sampling calls it schedules.

\paragraph{Joint predictions.}
At a selected boundary $j$, new predictions start immediately after the
committed prefix. Let $\mathcal H_j$ be the conditioning state derived from
that prefix, and let $c$ denote the current condition.
One DiT sampling call produces a
$K$-latent prediction $\widetilde Y_j$:
\begin{equation}
 \widetilde Y_j\sim p_\theta(\,\cdot\mid\mathcal H_j,c\,),
 \qquad |\widetilde Y_j|=K.
 \label{eq:sample}
\end{equation}
The next latent and the padding latents are sampled jointly with bidirectional
temporal attention and stored in the prediction cache. For a block-wise AR model,
$\mathcal H_j$ includes prefix KV; Section~\ref{sec:bidirectional} defines
an alternative for bidirectional clip models. Completed latents require
the appropriate VAE decoding before pixel delivery.

\paragraph{Generation and reuse.}
The controller has two states: $\mathsf{Produce}$, while a generation job
is running, and $\mathsf{Consume}$, while no job is running and cached
predictions can be reused. Playback continues in both states. At each
control-clock tick, the controller checks job progress and the latest
pending condition, then decides whether to generate or reuse predictions.
New conditions are saved for the next call without changing an active job;
if several arrive, only the latest is used.

\paragraph{Updating the uncommitted suffix.}
For a condition update, the controller chooses the earliest boundary after
the committed prefix where the inputs for continuation are available and
the new video can be ready in time for playback. It commits the latents
before this boundary, retaining enough existing video to play during
generation, and replaces the uncommitted suffix with a new complete chunk.
VAE continuation preserves the already committed output. The new chunk is
not shortened to match the number of discarded latents. If no suitable
boundary is available yet, the update waits and existing predictions remain
available for playback. Without an update, the controller appends another
chunk before the cached video runs out, allowing enough time for generation.
Figure~\ref{fig:scheduler} illustrates these actions,
Algorithm~\ref{alg:controller} gives the controller loop, and
Appendix~\ref{app:algorithm} gives the full state transitions and timing checks.

\begin{figure}[t]
  \centering
  \begin{minipage}{0.9\linewidth}
  \centering
  \includegraphics[width=\linewidth]{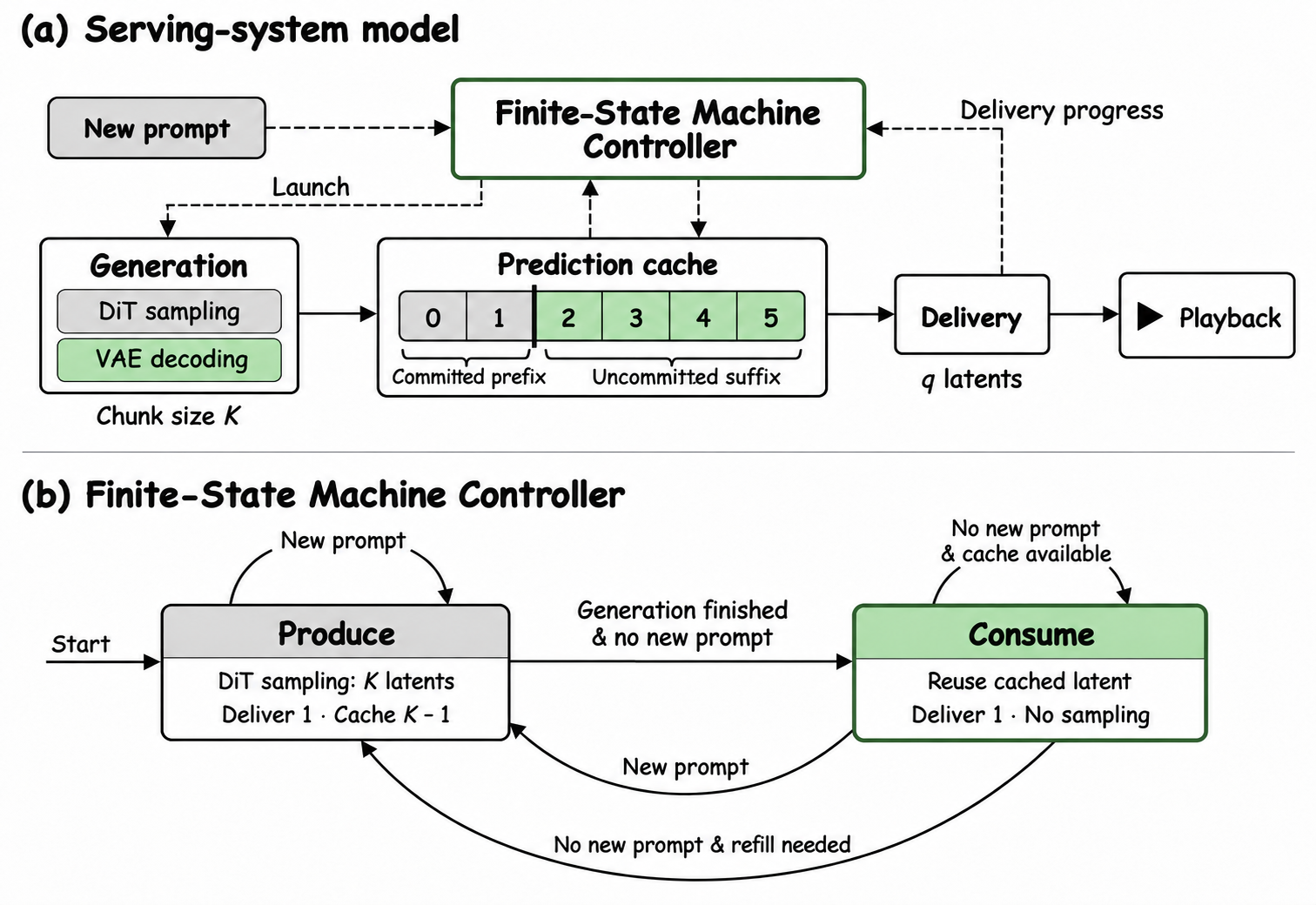}
  \end{minipage}
  \caption{\method{} serving system and control.
  (a) Joint generation fills a prediction cache for incremental delivery
  and playback. (b) The controller updates its Produce/Consume state at
  each control-clock tick, illustrated for $q=1$. Condition updates replace
  the uncommitted suffix without interrupting active generation;
  playback continues in both states.}
  \label{fig:scheduler}
\end{figure}

\subsection{Adapting Bidirectional Models for Streaming}
\label{sec:bidirectional}

Bidirectional teacher models often define the quality target for causal
streaming students, yet their full-sequence generation interfaces are not
designed for interactive serving. Directly converting the same backbone
to causal generation can also introduce a quality gap without careful
adaptation \citep{yin2024causvid}. This motivates a further extrapolation
of chunk size: \emph{can a large bidirectional teacher support low-latency
autoregressive generation while preserving its generation quality, without
Self-Forcing-style causal adaptation?} We preserve bidirectional joint DiT sampling within
each call and adapt its conditioning and output interfaces, allowing
successive calls to continue from committed video while outputs are
committed incrementally.

\paragraph{Continuation interface.}
Bidirectional video generators commonly provide native image- or
video-conditioned continuation interfaces, allowing successive jointly
sampled chunks to be connected without modifying within-chunk attention.
A practical complication arises with
temporally compressive 3D VAEs: the continuation interface often expects a
temporally uncompressed first-frame latent, whereas an interior cached frame
belongs to a temporally compressed latent. We therefore decode the
last latent of the committed prefix to pixels and re-encode its final frame through the VAE's
first-frame path before the next DiT sampling call. This produces the latent
format expected by the pretrained continuation interface. Alternatively,
the bidirectional model can be fine-tuned to accept compressed latent
continuation directly.

\subsection{Latency and Recomputation Analysis}
\label{sec:analysis}

\paragraph{Removing coarse playback alignment.}
Consider latents of equal video duration, fixed generation time $C$, and an isolated
update uniformly distributed over a chunk period $T=K\delta$.
If new output can enter playback only at chunk boundaries, then
\begin{equation}
 t_{\mathrm{playback}}^{\min}
 =T\left\lceil\frac{t_{\mathrm{arrival}}+C}{T}\right\rceil,
 \qquad \E[L]\geq C+\frac{T}{2}.
 \label{eq:chunk-bound}
\end{equation}
Since $L=R+C+Q$ and $C$ is fixed, this also gives $\E[R+Q]\geq T/2$.
This bounds the mean random-arrival waiting;
Appendix~\ref{app:latency} gives the derivation.
This boundary-imposed waiting persists whether it occurs before or after
generation. With the committed prefix ready for continuation and playback,
no additional buffering constraint,
and fine-grained commitment, the earliest latent-aligned display is
\begin{equation}
 \begin{aligned}
 t_{\mathrm{playback}}
 &=\delta\left\lceil\frac{t_{\mathrm{start}}+C}{\delta}\right\rceil,\\
 Q&=t_{\mathrm{playback}}-(t_{\mathrm{start}}+C),\qquad 0\leq Q<\delta.
 \end{aligned}
 \label{eq:switch}
\end{equation}
Thus alignment is tied to delivery granularity rather than chunk size.

\paragraph{Generation time and random-arrival waiting.}
  An arriving update may wait for the active generation job to complete
  the prefix needed for continuation.
  For one indivisible $C$-second job every $T$ seconds, with $C\leq T$,
  an isolated uniform arrival is blocked with probability $C/T$ and waits
  $C/2$ on average when blocked, provided that job supplies the required
  prefix. Hence
  \begin{equation}
   \E[R]=\frac{C^2}{2T},\qquad
   \E[L]=C+\frac{C^2}{2T}+\E[Q].
   \label{eq:residual}
  \end{equation}
  Reducing $C$ lowers both generation time and this conditional residual
wait (Figure~\ref{fig:latency}). For nonuniform completion phases,
$\E[Q]$ is determined by the actual pipeline timing
(Appendix~\ref{app:latency}).
\begin{figure}[t]
  \centering
  \centering
  \includegraphics[width=0.95\linewidth]{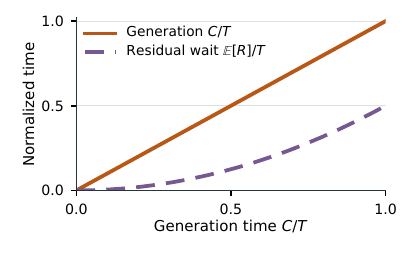}
  \caption{At fixed $T$, direct generation scales as $C/T$, while residual
  wait scales as $\E[R]/T=(C/T)^2/2$ under the stated periodic model.}
  \label{fig:latency}
\end{figure}

\paragraph{Cost of replacing the uncommitted suffix.}
Under unchanged conditions, cached predictions are delivered incrementally
and replenished chunk by chunk, using the same number of DiT sampling calls
as the matched chunk-wise baseline in steady state.
Updates introduce extra work by replacing unused predictions.
For equal-cost calls producing $K$ usable latents each, let $w$
be the fraction ultimately discarded rather than played. Relative DiT
sampling work satisfies
\begin{equation}
 A\simeq\frac{1}{1-w}.
 \label{eq:waste}
\end{equation}
Infrequent updates can keep this overhead small. The discarded fraction $w$
depends on update frequency, chunk size, and the selected boundary.
VAE decoding, adapter preparation, overlap trimming, aborted
jobs, and cache reconstruction require separate cost accounting.
Appendix~\ref{app:cost} gives the assumptions and a discrete update model.
Together, these results motivate accelerating high-quality joint
generation rather than shrinking chunks solely for responsiveness.
DiT sampling acceleration and parallel acceleration target the generation
path; their benefit is measured through $C$ and total resource cost.

\section{Experiments}
\label{sec:experiments}

We first characterize how chunk size affects generation quality and
computational efficiency within and across GPUs, then evaluate quality under
condition updates. Finally, we compare conventional chunk commitment and \method{} in end-to-end interactive serving.

\begin{figure}[t]
  \centering
  \includegraphics[width=\linewidth]{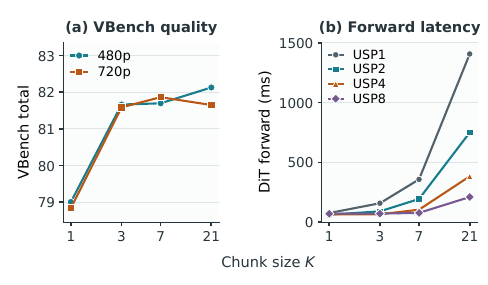}
  \caption{Chunk size exposes a quality--efficiency tradeoff across models.
  (a) Wan four-step VBench totals at two resolutions. (b) Wan H200 forward
  latency across USP sizes at 480p. Quality uses separately trained
  configurations; timings use a shared reference checkpoint and synthetic
  fixed-length history. LTX quality and cached-forward profiles are reported in
  Table~\ref{tab:model-profile}.}
  \label{fig:tradeoff}
\end{figure}

\begin{table}[t]
  \caption{Wan1.3B quality measurements. Total is the VBench aggregate,
  with quality and semantic components. Bold marks the best score per
  resolution and step count.}
  \label{tab:wan-quality}
  \centering
  \committablestyle
  \setlength{\tabcolsep}{1pt}
\begin{tabular*}{\columnwidth}{@{\extracolsep{\fill}}crrrrrrr@{}}
\toprule
\smash{\rlap{\color[HTML]{F7F3EC}\rule[-\dp\strutbox]{\columnwidth}{\dimexpr2\ht\strutbox+2\dp\strutbox-2pt\relax}}}%
Steps & $K$ & Total & Quality & Semantic & \shortstack{Dynamic\\degree} & \shortstack{Vision\\Reward} & \shortstack{Instruction\\following}\\
\midrule
\rowcolor[HTML]{F1F5F8}\multicolumn{8}{c}{\textbf{480p}}\\
\addlinespace[1pt]
\multirow{4}{*}{1} & 1 & 79.11 & 81.80 & 68.34 & 4.00 & 6.50 & 90.00\\
 & 3 & 80.38 & 83.40 & 68.29 & 58.00 & 4.81 & 94.00\\
 & 7 & 81.35 & 83.97 & \textbf{70.86} & 48.00 & 5.79 & \textbf{100.00}\\
 & 21 & \textbf{81.68} & \textbf{84.83} & 69.08 & \textbf{85.00} & \textbf{9.93} & 94.00\\
\addlinespace[2pt]
\multirow{4}{*}{2} & 1 & 80.07 & 82.76 & 69.34 & 14.00 & 9.31 & 98.00\\
 & 3 & 81.62 & \textbf{84.63} & 69.57 & 57.00 & 9.82 & \textbf{100.00}\\
 & 7 & 80.84 & 84.06 & 67.94 & \textbf{93.00} & 9.29 & \textbf{100.00}\\
 & 21 & \textbf{81.84} & 84.62 & \textbf{70.74} & 71.00 & \textbf{10.17} & \textbf{100.00}\\
\addlinespace[2pt]
\multirow{4}{*}{4} & 1 & 79.00 & 81.38 & 69.50 & 13.00 & 7.65 & 96.00\\
 & 3 & 81.66 & 84.70 & 69.51 & 80.00 & 9.64 & \textbf{100.00}\\
 & 7 & 81.70 & 84.73 & 69.56 & 78.00 & \textbf{10.19} & 98.00\\
 & 21 & \textbf{82.13} & \textbf{85.03} & \textbf{70.54} & \textbf{85.00} & 9.66 & 98.00\\
\midrule
\rowcolor[HTML]{F1F5F8}\multicolumn{8}{c}{\textbf{720p}}\\
\addlinespace[1pt]
\multirow{4}{*}{1} & 1 & 78.94 & 81.37 & 69.19 & 4.00 & 5.82 & 92.00\\
 & 3 & 79.99 & 82.70 & 69.18 & 47.00 & 3.93 & 98.00\\
 & 7 & \textbf{80.78} & 83.05 & \textbf{71.67} & 41.00 & 5.23 & 96.00\\
 & 21 & 80.49 & \textbf{83.10} & 70.04 & \textbf{67.00} & \textbf{9.56} & \textbf{100.00}\\
\addlinespace[2pt]
\multirow{4}{*}{2} & 1 & 79.86 & 82.40 & 69.70 & 9.00 & 9.43 & 98.00\\
 & 3 & \textbf{81.52} & \textbf{84.11} & \textbf{71.18} & 48.00 & \textbf{10.10} & \textbf{100.00}\\
 & 7 & 80.97 & 83.79 & 69.67 & \textbf{84.00} & 9.25 & 94.00\\
 & 21 & 80.26 & 82.61 & 70.84 & 39.00 & 9.91 & 98.00\\
\addlinespace[2pt]
\multirow{4}{*}{4} & 1 & 78.84 & 81.19 & 69.44 & 22.00 & 7.18 & 92.00\\
 & 3 & 81.59 & 84.36 & 70.49 & \textbf{74.00} & 9.59 & \textbf{100.00}\\
 & 7 & \textbf{81.86} & \textbf{84.62} & \textbf{70.83} & 73.00 & \textbf{9.77} & 96.00\\
 & 21 & 81.65 & 84.43 & 70.52 & 66.00 & 9.36 & 96.00\\
\bottomrule
\end{tabular*}

\end{table}

\begin{table*}[t]
  \caption{Generation quality and DiT forward latency on H200 GPUs.
  Quality scores are multiplied by 100. Bold marks the lowest latency per row
  and highest quality per column. Profiling details are in
  Appendix~\ref{app:experiment-details}.}
  \label{tab:model-profile}
  \centering
  \committablestyle
  \newsavebox{\profileWanBox}
  \newsavebox{\profileLtxBox}
  \sbox{\profileWanBox}{%
    \setlength{\tabcolsep}{1.0pt}
    \begin{tabular}{@{}c@{}}
      \textbf{(a) Wan1.3B}\\[1pt]
      \begin{tabular}{llrrrr}
\toprule
\rowcolor[HTML]{F7F3EC}
\multicolumn{6}{c}{DiT forward latency (ms, $\downarrow$)}\\
\midrule
\rowcolor[HTML]{F7F3EC}
Resolution & $K$ & USP1 & USP2 & USP4 & USP8\\
\midrule
480p & 1 & 77.5 & \textbf{63.7} & 65.9 & 72.1\\
 & 3 & 159.2 & 91.5 & \textbf{67.0} & 72.9\\
 & 7 & 358.3 & 194.7 & 105.2 & \textbf{78.6}\\
 & 21 & 1406.2 & 750.0 & 384.1 & \textbf{211.0}\\
\addlinespace
720p & 1 & 218.7 & 143.4 & 95.8 & \textbf{72.7}\\
 & 3 & 599.6 & 318.4 & 166.5 & \textbf{119.4}\\
 & 7 & 1521.4 & 809.2 & 443.8 & \textbf{276.4}\\
 & 21 & 6432.8 & 3335.5 & 1737.1 & \textbf{899.7}\\
\bottomrule
\end{tabular}

    \end{tabular}%
  }
  \sbox{\profileLtxBox}{%
    \setlength{\tabcolsep}{1.65pt}
    \begin{tabular}{@{}c@{}}
      \textbf{(b) LTX-2.3}\\[1pt]
      \begin{tabular}{lrrrrrrrrrr}
\toprule
\rowcolor[HTML]{F7F3EC}
& \multicolumn{6}{c}{Generation quality ($\uparrow$)} & \multicolumn{4}{c}{DiT forward latency (ms, $\downarrow$)}\\
\cmidrule(lr){2-7}\cmidrule(lr){8-11}
\rowcolor[HTML]{F7F3EC}
$K$ & \shortstack{Subject\\consistency} & \shortstack{Background\\consistency} & \shortstack{Aesthetic\\quality} & \shortstack{Imaging\\quality} & \shortstack{Motion\\smoothness} & \shortstack{Dynamic\\degree} & USP1 & USP2 & USP4 & USP8\\
\midrule
1 & 87.97 & \textbf{94.68} & \textbf{52.82} & \textbf{73.09} & 99.62 & 31.25 & \textbf{274.0} & 440.5 & 479.6 & 526.1\\
3 & 88.81 & 94.31 & 52.21 & 65.98 & 99.53 & 46.88 & \textbf{255.6} & 423.5 & 457.0 & 511.2\\
9 & \textbf{89.42} & 93.27 & 50.72 & 68.96 & 99.56 & 34.38 & 426.7 & \textbf{417.8} & 456.6 & 502.0\\
15 & 84.96 & 92.21 & 51.01 & 65.98 & 99.50 & \textbf{78.12} & 660.3 & \textbf{420.4} & 461.0 & 509.4\\
45 & 88.72 & 93.68 & 52.40 & 68.75 & \textbf{99.64} & 40.62 & 2016.2 & 1158.4 & 634.0 & \textbf{507.9}\\
\bottomrule
\end{tabular}

    \end{tabular}%
  }
  \resizebox{\textwidth}{!}{%
    \raisebox{\depth}{\resizebox*{!}{100pt}{\usebox{\profileWanBox}}}%
    \raisebox{\depth}{\resizebox*{!}{100pt}{\usebox{\profileLtxBox}}}%
  }
\end{table*}

\subsection{Experimental Setup}
\label{sec:exp-setup}

\paragraph{Models and quality evaluation.}
For Wan1.3B, we follow CF++ \citep{zhu2026causalpp} and train on Toy6K,
$K\in\{1,3,7,21\}$ models with 1/2/4-step sampling. VBench
\citep{huang2024vbench} evaluates 81-frame, 16-fps videos at $832\times480$
and $1280\times720$ over 946 prompts and five seeds. For LTX-2.3
\citep{hacohen2026ltx2}, we train $K\in\{1,3,9,15,45\}$ models with three DiT sampling steps,
on an internal 8K-video training set (Internal8K). We evaluate on an internal
32-case set (Internal32), with matched seeds and 361-frame outputs at
$832\times448$ and 24 fps. We report its six VBench visual scores separately.

\paragraph{Condition updates.}
We construct 30 image-to-video test cases by pairing text prompts from
VBench~\citep{huang2024vbench} with generated first frames, adapting text-only
evaluation inputs to our models' image-conditioned interface. We evaluate a
fixed LTX checkpoint with $K=9$ at condition update probabilities of
0\%, 35\%, 70\%, and 100\%,
alongside a native $K=1$ model, using the same image--prompt pairs across configurations.
Update events are sampled before inference at latent boundaries. We keep the
prompt text unchanged to isolate the effect of discarding and regenerating
predictions.
At 100\%, each call generates up to nine latents but commits only one,
discarding the rest.

\paragraph{Human evaluation.}
We also evaluate 15 LTX cases with conflicting action-prompt updates,
comparing ours and the large baseline ($K=9$) with the small baseline ($K=1$).
Ten raters each assess all 15 cases for new-action response and visual
stability (continuity, flicker, and abrupt jumps). We aggregate pairwise
Good/Same/Bad judgments as GSB ratios (Table~\ref{tab:prompt-switch-human}).

\paragraph{Forward latency.}
We measure the latency of a single DiT forward, including inter-GPU
communication, with USP on 1, 2, 4, and 8 H200 GPUs.

\paragraph{Interactive serving.}
We compare chunk-level commitment (Baseline) with \method{} on Wan1.3B,
LTX-2.3, and Sol-H3~\citep{solh3_2026,li2026solvideoinferenceengine}, using
small and large chunks for each model. Serving uses 1--8 B300 GPUs for DiT
inference at approximately 480p and 24 fps. Within each configuration,
baseline and our method use the same hardware and generation pipeline.
For LTX, we also evaluate CUDA Graph acceleration of the small-chunk model.
Model-specific chunk sizes, GPU allocations, and execution settings are
listed in Appendix~\ref{app:experiment-details}.

Each configuration receives the same 30 one-minute interaction sequences,
with low-frequency (1--2/min) or high-frequency (3--11/min) condition updates. Updates change actions
within a scene and arrive independently of playback boundaries. The baseline
applies new conditions at the next chunk refill; our method can replace
uncommitted predictions.

\paragraph{Serving metrics.}
We report median and 95th-percentile (P95) response latency, measured from condition
arrival to actual display. GPU active time is the percentage of time monitored
GPUs execute kernels; playback stall is the percentage of the session spent
waiting for video. Measurement details and monitored GPUs are specified in
Appendix~\ref{app:experiment-details}.

\subsection{Impact of Chunk Size}
\label{sec:model-profiling}

\paragraph{Larger chunks improve quality.}
Wan benefits most from moving beyond single-latent generation
(Table~\ref{tab:wan-quality}). At four steps and 480p, increasing $K$ from
1 to 3 raises the VBench total from 79.00 to 81.66; increasing it further
to 21 yields 82.13. The best chunk size varies across settings, but the
smallest chunk scores lower than larger alternatives.

\paragraph{Parallel acceleration limits the latency cost of larger chunks.}
More generated latents need not imply proportionally longer DiT
forwards (Table~\ref{tab:model-profile}). For Wan at 480p, $K=7$ with USP8
takes 78.6\,ms, nearly matching 77.5\,ms for $K=1$ with USP1. For LTX,
increasing $K$ from 1 with USP1 to 45 with USP8 raises forward latency
from 274.0 to 507.9\,ms, less than twofold for 45 times as many latents.

Together, these results motivate choosing chunk size for quality and
per-chunk generation time, rather than minimizing it by default. We next
test whether large-chunk quality survives condition updates and whether
large chunks can support responsive interaction.

\begin{figure*}[!b]
  \centering
  \includegraphics[width=0.72\textwidth]{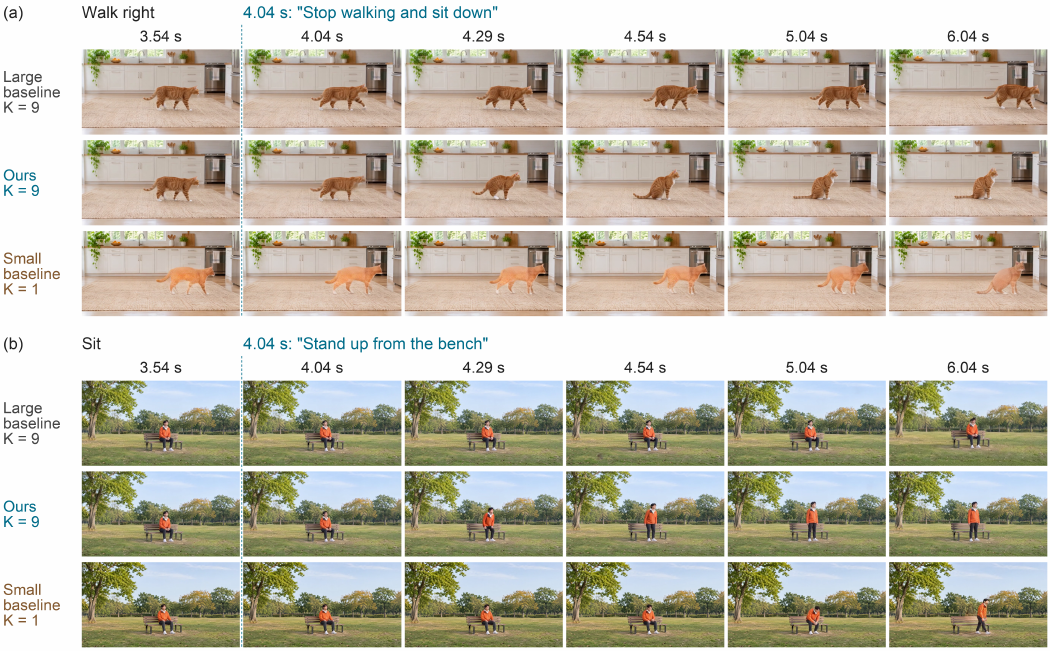}
  \caption{LTX action updates at 4.04\,s: (a) walking to sitting;
  (b) sitting to standing. By 4.54\,s, ours shows the requested pose in both cases.
  Columns show uncropped frames at matched video times.}
  \label{fig:prompt-switch}
\end{figure*}

\subsection{Generation Quality under Condition Updates}
\label{sec:interruption-quality}

\begin{table}[!htbp]
  \caption{LTX quality versus condition update probability: means over 30 full videos,
  $\times100$. Chunk size 1 is the native baseline.}
  \label{tab:ltx-interruption}
  \centering
  \committablestyle
  \setlength{\tabcolsep}{1pt}
\begin{tabular*}{\columnwidth}{@{\extracolsep{\fill}}ccrrrrr@{}}
\toprule
\smash{\rlap{\color[HTML]{F7F3EC}\rule[-\dp\strutbox]{\columnwidth}{\dimexpr2\ht\strutbox+2\dp\strutbox-2pt\relax}}}%
\shortstack{Chunk\\size $K$} & \shortstack{Update\\probability} & \shortstack{Aesthetic\\quality} & \shortstack{Imaging\\quality} & \shortstack{Dynamic\\degree} & \shortstack{Vision\\Reward} & \shortstack{Instruction\\Following}\\
\midrule
9 & 0\% & 63.31 & 73.58 & 86.67 & 15.11 & 100.00\\
9 & 35\% & 63.37 & 73.87 & 96.67 & 14.61 & 100.00\\
9 & 70\% & 62.29 & 73.17 & 96.67 & 12.15 & 100.00\\
9 & 100\% & 61.53 & 73.23 & 96.67 & 9.25 & 100.00\\
\midrule
1 & 0\% & 60.53 & 72.33 & 3.33 & 6.05 & 100.00\\
\bottomrule
\end{tabular*}

\end{table}

\begin{table*}[!t]
  \caption{Interactive serving performance at 480p-class resolution.
  Baseline commits complete chunks; Ours uses \method{}.
  CUDA Graph accelerates LTX DiT execution; other LTX rows use eager execution.
  Lower response latency and playback stall are better.}
  \label{tab:main}
  \centering\footnotesize
  \renewcommand{\arraystretch}{0.96}
  \setlength{\tabcolsep}{1.3pt}
  \setlength{\aboverulesep}{0.25ex}
  \setlength{\belowrulesep}{0.25ex}
  \begin{tabular*}{\textwidth}{@{\extracolsep{\fill}}lcllrrrrrrrr@{}}
\toprule
\smash{\rlap{\color[HTML]{F7F3EC}\rule[-24.1pt]{\textwidth}{31.4pt}}}%
\multirow{2}{*}[-5pt]{Model} & \multirow{2}{*}[-5pt]{\shortstack{Chunk\\size}} & \multirow{2}{*}[-5pt]{Method} & \multirow{2}{*}[-5pt]{Execution} & \multicolumn{4}{c}{Low-frequency updates} & \multicolumn{4}{c}{High-frequency updates}\\
\cmidrule(lr){5-8}\cmidrule(lr){9-12}
& & & & \shortstack{Median\\latency (s)} & \shortstack{P95\\latency (s)} & \shortstack{GPU active\\time (\%)} & \shortstack{Playback\\stall (\%)} & \shortstack{Median\\latency (s)} & \shortstack{P95\\latency (s)} & \shortstack{GPU active\\time (\%)} & \shortstack{Playback\\stall (\%)}\\
\midrule
Wan1.3B & Small & Baseline & Default & 1.13 & 1.47 & 49.67 & 0.000 & 1.24 & 1.65 & 52.19 & 0.000\\
 & Large & Baseline & Default & 3.23 & 4.64 & 16.97 & 0.000 & 3.11 & 4.66 & 17.29 & 0.000\\
 & Small & Ours & Default & 1.01 & 1.14 & 50.45 & 0.000 & 1.05 & 1.35 & 52.25 & 0.000\\
 & Large & Ours & Default & 1.32 & 1.91 & 17.28 & 0.000 & 1.35 & 1.63 & 19.92 & 0.000\\
\specialrule{0.4pt}{2.5pt}{2.5pt}
LTX-2.3 & Small & Baseline & Eager & 2.12 & 2.58 & 22.02 & 19.669 & 2.01 & 2.64 & 22.69 & 20.339\\
 & Small & Baseline & CUDA Graph & 1.81 & 2.32 & 29.19 & 0.005 & 1.89 & 2.34 & 30.30 & 0.021\\
 & Large & Baseline & Eager & 4.82 & 6.99 & 20.47 & 0.069 & 5.09 & 6.98 & 21.41 & 0.065\\
 & Small & Ours & Eager & 2.00 & 2.67 & 21.94 & 19.758 & 2.00 & 2.57 & 22.76 & 19.878\\
 & Small & Ours & CUDA Graph & 2.01 & 2.82 & 29.03 & 1.781 & 1.94 & 2.54 & 29.99 & 4.407\\
 & Large & Ours & Eager & 2.39 & 4.13 & 21.29 & 0.204 & 2.39 & 4.24 & 24.51 & 0.524\\
\specialrule{0.4pt}{2.5pt}{2.5pt}
Sol-H3 & Small & Baseline & Default & 1.87 & 2.11 & 48.62 & 0.000 & 1.87 & 2.04 & 49.10 & 0.000\\
 & Large & Baseline & Default & 4.49 & 6.44 & 17.42 & 0.000 & 4.45 & 6.59 & 18.25 & 0.000\\
 & Small & Ours & Default & 1.11 & 1.50 & 48.85 & 0.000 & 1.15 & 1.53 & 50.00 & 0.000\\
 & Large & Ours & Default & 1.42 & 2.47 & 18.45 & 0.000 & 1.41 & 1.95 & 22.67 & 0.000\\
\bottomrule
\end{tabular*}

\end{table*}

Table~\ref{tab:ltx-interruption} compares Speculative Sampling against
uninterrupted large-chunk generation (0\% updates) and native single-latent
generation. As update probability increases, imaging quality remains stable
and aesthetic quality decreases only slightly, indicating that frequent updates
largely preserve visual appearance. VisionReward declines, but even at 100\%
updates, the $K=9$ model retains higher VisionReward and substantially higher
dynamic degree than native $K=1$, with slightly higher aesthetic and imaging
scores. Thus, Speculative Sampling enables single-latent commitment while
retaining quality benefits from large-chunk generation, even without prediction
reuse. This suggests
inference-time scaling through extra computation on the uncommitted suffix, with potential applications in video editing and keyboard-controlled
world models.

\paragraph{Human evaluation under action updates.}
Table~\ref{tab:prompt-switch-human} shows that ours receives better response
but weaker stability ratings than the large
baseline. Against the small baseline, response is mostly tied while stability
favors ours. Figure~\ref{fig:prompt-switch} shows two examples.

\begin{table}[t]
  \caption{Good/Same/Bad (GSB) ratios $(G+S)/(B+S)$~\citep{lin2025contentv}; values above 1
  favor the first method. Ours and large baseline use $K=9$; small baseline uses $K=1$.}
  \label{tab:prompt-switch-human}
  \centering
  \committablestyle
  \begin{tabular*}{\columnwidth}{@{\extracolsep{\fill}}lrr@{}}
    \toprule
    \smash{\rlap{\color[HTML]{F7F3EC}\rule[-\dp\strutbox]{\columnwidth}{\dimexpr2\ht\strutbox+2\dp\strutbox-2pt\relax}}}%
Comparison & \shortstack{Action\\response} & \shortstack{Visual\\stability}\\
    \midrule
    Ours vs. large baseline & 8.33 & 0.77\\
    Ours vs. small baseline & 0.93 & 13.36\\
    Small baseline vs. large baseline & 5.77 & 0.01\\
    \bottomrule
  \end{tabular*}
\end{table}

\subsection{Interactive Serving Performance}
\label{sec:serving-evaluation}

\paragraph{Our method reduces response latency for large chunks.}
Table~\ref{tab:main} shows lower large-chunk response latency across all three
models. For Wan under 3--11 condition updates per minute, median latency falls
from 3.11\,s to 1.35\,s. The measured mean reductions also follow the trend
predicted by our chunk-boundary analysis (Appendix~\ref{app:experiment-details}):
larger chunks offer more boundary waiting to eliminate.

\paragraph{Our method achieves similar latency across chunk sizes.}
With our method, large-chunk median latency is only 0.26--0.39\,s above the
corresponding small-chunk configuration across all three models and both
update frequencies, comparing eager execution for LTX.
Large chunks also show lower average GPU active time for Wan and Sol-H3,
and comparable activity for LTX. Under 3--11 updates per minute, large-chunk
GPU active time stays within 20--25\% across the three deployments.

\paragraph{Our method lets large chunks respond faster than small-chunk baselines.}
For Sol-H3 under 3--11 updates per minute, our large-chunk configuration
achieves a median response latency of 1.41\,s, compared with 1.87\,s for the
small-chunk baseline, with no playback stalls. Retaining a larger joint
generation window can therefore be compatible with prompt interaction.

\paragraph{Implications for parallel serving.}
Larger chunks exploit parallel execution to jointly process more video
latents with only a modest increase in DiT forward latency
(Table~\ref{tab:model-profile}). Consistent with this efficiency benefit,
Table~\ref{tab:main} shows substantially lower GPU active time for larger
chunks in Wan and Sol-H3, suggesting the potential to support more concurrent
interactive video sessions on the same GPU resources. Speculative Sampling
brings large-chunk interaction latency close to that of small chunks, making
large chunks a practical option for balancing response latency, generation
quality, and deployment efficiency.

\section{Discussion and Conclusion}
\label{sec:limits}
\label{sec:conclusion}

We have formulated interactive video generation as a joint problem of
prediction and output commitment. Our method preserves multi-latent joint generation,
reuses completed predictions under unchanged conditions, and replaces only
the uncommitted suffix while preserving the committed prefix for continued
playback. This shifts part of the design effort from shrinking the generation
unit toward accelerating joint generation, managing the uncommitted suffix,
and scheduling timely delivery.
Our condition-update study further suggests a promising direction for inference-time
scaling in interactive video generation: allocating additional computation to
future latents can improve delivered outputs even when those future predictions
are discarded.

Frequent updates reduce prediction reuse and introduce additional DiT sampling work.
Our results nevertheless make large chunks a practical option for interactive
video generation, allowing chunk size to be chosen by balancing generation
quality, response latency, and deployment efficiency.

\label{main-text-end}
\label{post-body-start}
\section*{LLM Usage Statement}
OpenAI Codex (a GPT-based coding and writing assistant) assisted with
translating and developing author-provided research notes, organizing and
editing the manuscript, refining definitions and experimental protocols,
checking mathematical arguments, locating primary-source references,
implementing and running experimental scripts, and preparing LaTeX,
figures, and document-validation code. The authors are responsible for
verifying all claims, references, code, and experimental results in the
submission.

\label{references-start}
{\small
\bibliographystyle{ieeenat_fullname}
\bibliography{references}
}
\clearpage
\label{supplement-start}
\appendix
\twocolumn[{
  \begin{center}
    {\Large\textbf{\thetitle}\\[0.5em]Supplementary Material}
  \end{center}
  \begin{minipage}{\textwidth}
    \begingroup
  \centering
  \footnotesize
  \setlength{\tabcolsep}{2.3pt}
  \renewcommand{\arraystretch}{1.12}
  \begin{tabular}{lccccrrrrrrrr}
    \toprule
    Method & \shortstack{Chunk\\size} & \shortstack{Denoising\\window} & \shortstack{Interruption\\rate} & \shortstack{\textbf{Avg.}\\\textbf{Quality}$\uparrow$} & \shortstack{Subject\\cons.$\uparrow$} & \shortstack{Bkgd.\\cons.$\uparrow$} & \shortstack{Flicker\\score$\uparrow$} & \shortstack{Motion\\smooth.$\uparrow$} & \shortstack{Aesthetic\\quality$\uparrow$} & \shortstack{Imaging\\quality$\uparrow$} & \shortstack{Dynamic\\degree$\uparrow$} & \shortstack{Imaging\\drift$\downarrow$} \\
    \midrule
    Rolling Forcing & 3 & 15 & 0\% & 75.91 & 86.76 & 89.78 & 96.51 & 98.51 & 57.81 & 70.62 & 16.00 & 4.10 \\
    Self Forcing & 3 & 3 & 0\% & 75.05 & 82.10 & 88.61 & 98.22 & 98.76 & 56.54 & 71.35 & 9.00 & 4.98 \\
    Self Forcing & 15 & 15 & 0\% & \textbf{80.63} & 92.76 & 94.17 & 98.11 & 98.98 & 61.05 & 69.66 & 35.00 & 2.73 \\
    \midrule
    Self Forcing + ours & 15 & 15 & 35\% & 80.37 & 90.61 & 93.05 & 97.41 & 98.83 & 60.09 & 68.84 & 48.00 & 3.10 \\
    Self Forcing + ours & 15 & 15 & 70\% & 80.55 & 88.57 & 91.94 & 96.79 & 98.51 & 59.24 & 67.28 & 68.50 & 3.81 \\
    Self Forcing + ours & 15 & 15 & 100\% & 80.08 & 86.31 & 90.68 & 96.17 & 98.02 & 58.63 & 65.06 & 83.50 & 5.77 \\
    \bottomrule
  \end{tabular}
  \captionof{table}{\textbf{Rolling-window and full-chunk generation on MovieGen200.} Each row evaluates 200 complete 30-second videos. Sizes are in temporal latents; the denoising window counts active latents and excludes clean history. RF advances its window by three latents; our method uses the same $K=15$ checkpoint and commits one latent per call at 100\% interruption. \textbf{Avg. Quality} is the VBench weighted quality average; all quality scores are multiplied by 100. Imaging drift is evaluated separately.}
  \label{tab:rf-interruption}
\endgroup

  \end{minipage}
  \vspace{1em}
}]

\raggedbottom
\section{Supplementary Experiments}
\label{app:experiments}
\subsection{Comparison with Rolling Forcing}
\label{app:rolling-forcing}

Rolling Forcing (RF)~\cite{liu2025rolling} is closely related to our goal of combining a large denoising window with fine-grained interaction. Its rolling window processes fifteen latents at different noise levels while admitting new conditions and emitting completed outputs in three-latent increments. Both approaches therefore seek to retain a large modeling window while allowing finer-grained input and output. We compare them under matched training and evaluation settings to examine the quality obtained from rolling-window and full-chunk joint DiT sampling, and how much of the latter's benefit survives interrupted commitment.

\paragraph{Matched comparison.}
We train RF and two Self Forcing~\cite{huang2025self} variants with chunk sizes $K=3$ and $K=15$ using the same initialization, training prompts, optimizer settings, and 3,000-iteration budget. All three use a 30-latent training rollout, a 21-latent attention window including clean history, and a three-latent attention sink. RF retains its mixture of rolling and three-latent Self Forcing rollouts; the two Self Forcing variants use only their respective chunk-wise rollouts. We evaluate the final EMA checkpoints with five DiT sampling steps on the 200 MovieGen prompts listed by RF, with matched per-prompt seeds, at 30 seconds, 16 fps, and $832\times480$ resolution.

\paragraph{Quality of the large window.}
Table~\ref{tab:rf-interruption} reports an average quality score using the official VBench~\cite{huang2024vbench} normalization and weights for its seven quality dimensions. Full-chunk Self Forcing with $K=15$ reaches \textbf{80.63}, compared with \textbf{75.91} for RF and \textbf{75.05} for $K=3$. Thus, in this matched comparison, the large rolling window provides a smaller aggregate quality gain than jointly sampling a complete fifteen-latent chunk. The full-chunk model improves subject and background consistency, aesthetic quality, and motion, with lower first-to-last imaging drift.

\paragraph{Retaining quality under frame-wise commitment.}
We then apply \method{} to the same $K=15$ checkpoint with interruption rates of 35\%, 70\%, and 100\%. Each eligible latent boundary independently triggers replacement with the specified probability; the trajectories are sampled before inference. Prompts remain fixed throughout each video, isolating the effect of discarding and regenerating the uncommitted suffix. At 100\%, every call samples fifteen latents but commits only one; all remaining predictions are discarded. Average quality stays between \textbf{80.08 and 80.63} across all four rates. Even at 100\%, it exceeds RF by \textbf{4.17} points and $K=3$ by \textbf{5.03} points. The quality benefit of joint large-chunk generation therefore persists under frame-wise commitment, even when its uncommitted predictions are never reused.

More frequent interruptions increase dynamic degree while reducing consistency, aesthetic and imaging quality, and temporal stability. These changes largely offset one another in the aggregate score. Figures~\ref{fig:rf-interruption-chef}--\ref{fig:rf-interruption-rally} illustrate these trade-offs on four prompts. Together with the main paper's serving results, this experiment supports separating the modeling window from the commitment granularity to preserve large-chunk quality during fine-grained interaction.

\twocolumn[{
\begin{minipage}{\textwidth}
  \centering
  \includegraphics[width=\linewidth]{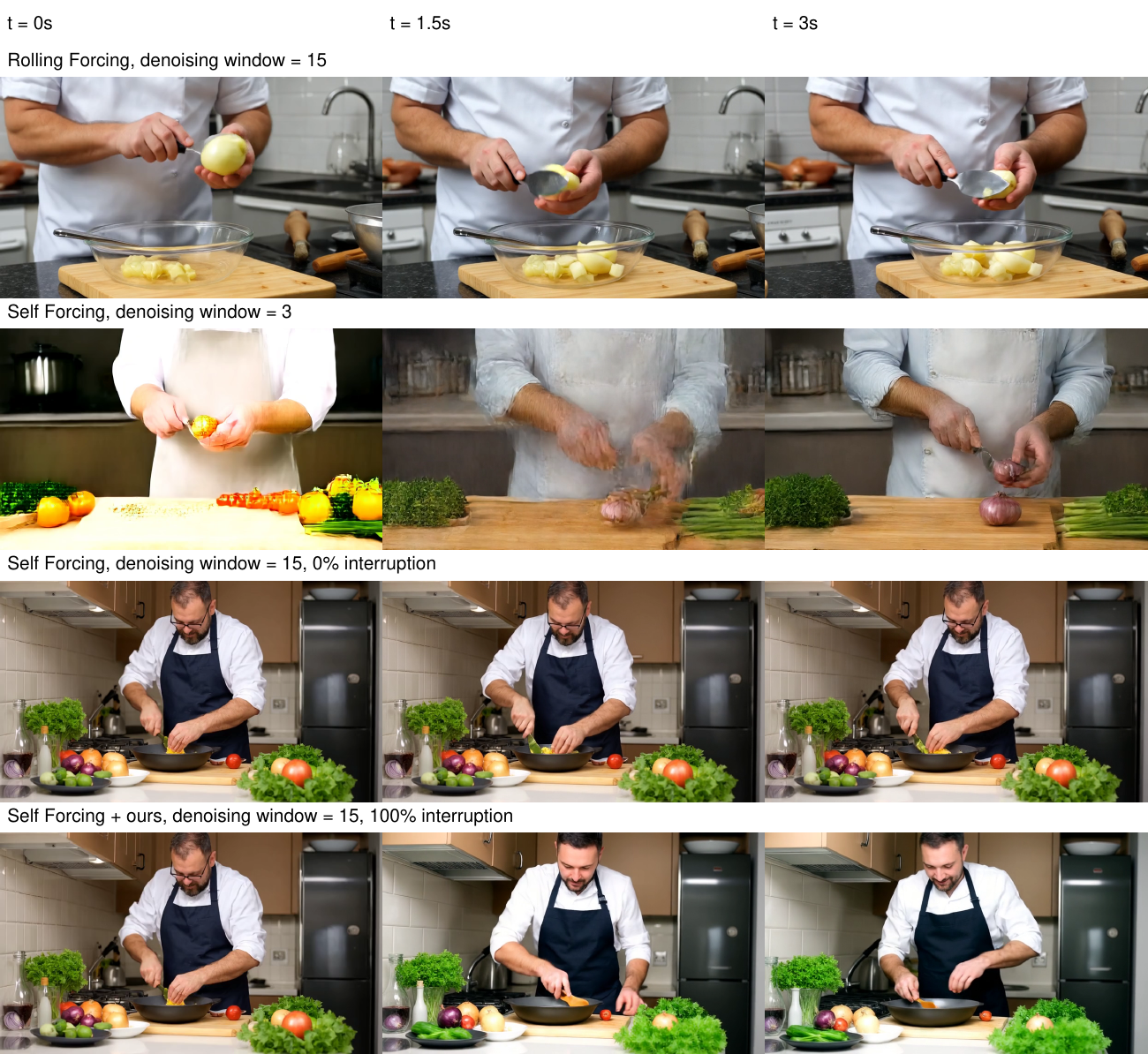}
  \captionof{figure}{\textbf{Chopping onions.} RF and large-chunk sampling improve visual clarity over SF-K3; our frame-wise outputs show more pronounced motion.}
  \label{fig:rf-interruption-chef}
\end{minipage}
\vspace{1em}
}]

\paragraph{Qualitative comparison.}
The four cases compare RF, Self Forcing with chunk size three (SF-K3), and Self Forcing with chunk size fifteen (SF-K15) at 0\% and 100\% interruption, using the same per-case prompt and seed. Frames are shown at 0, 1.5, and 3 seconds; the supplementary videos show the complete 30-second sequences. At 100\%, our method samples fifteen latents and commits one per call.

Both RF and our large-chunk approach improve on small-chunk SF, with uninterrupted $K=15$ giving the best visual quality in these examples. Over the full sequences, frequent interruption leads to faster error accumulation, consistent with the increased number of autoregressive steps needed to deliver the same duration. RF exhibits less apparent error accumulation but substantially weaker dynamics. Our interrupted outputs retain richer motion and much of the large-chunk visual quality, illustrating the trade-off between motion and long-horizon stability.

\noindent\textbf{Prompt.} Chef chopping onions in the kitchen for the preparation of the dish.

\twocolumn[{
\begin{minipage}{\textwidth}
  \centering
  \includegraphics[width=\linewidth]{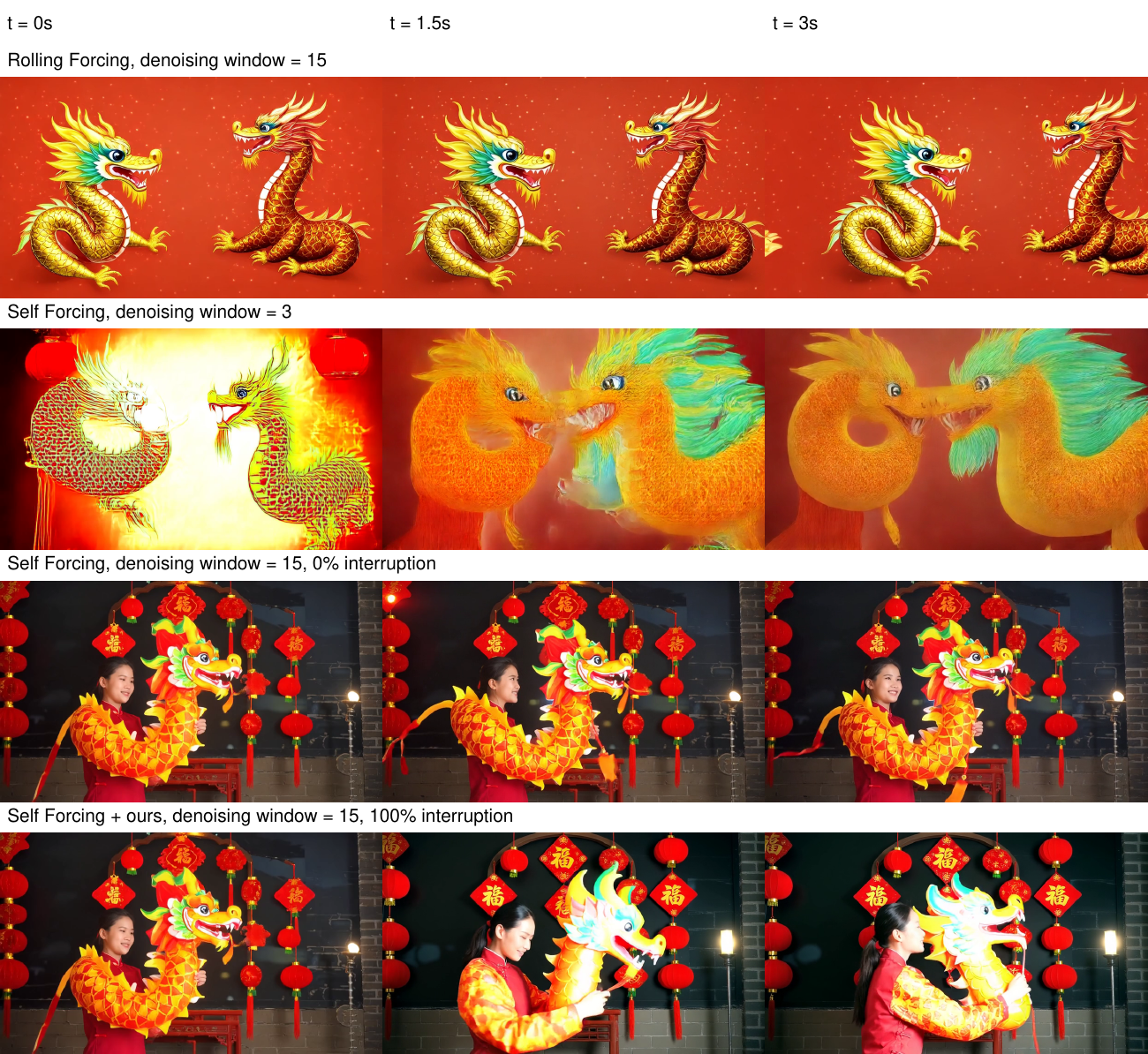}
  \captionof{figure}{\textbf{Dragon celebration.} Large-chunk sampling depicts a human-operated dragon, with more pronounced movement under frame-wise commitment.}
  \label{fig:rf-interruption-dragon}
\end{minipage}
\vspace{1em}
}]

\noindent\textbf{Prompt.} A Chinese Lunar New Year celebration video with Chinese Dragon.

\twocolumn[{
\begin{minipage}{\textwidth}
  \centering
  \includegraphics[width=\linewidth]{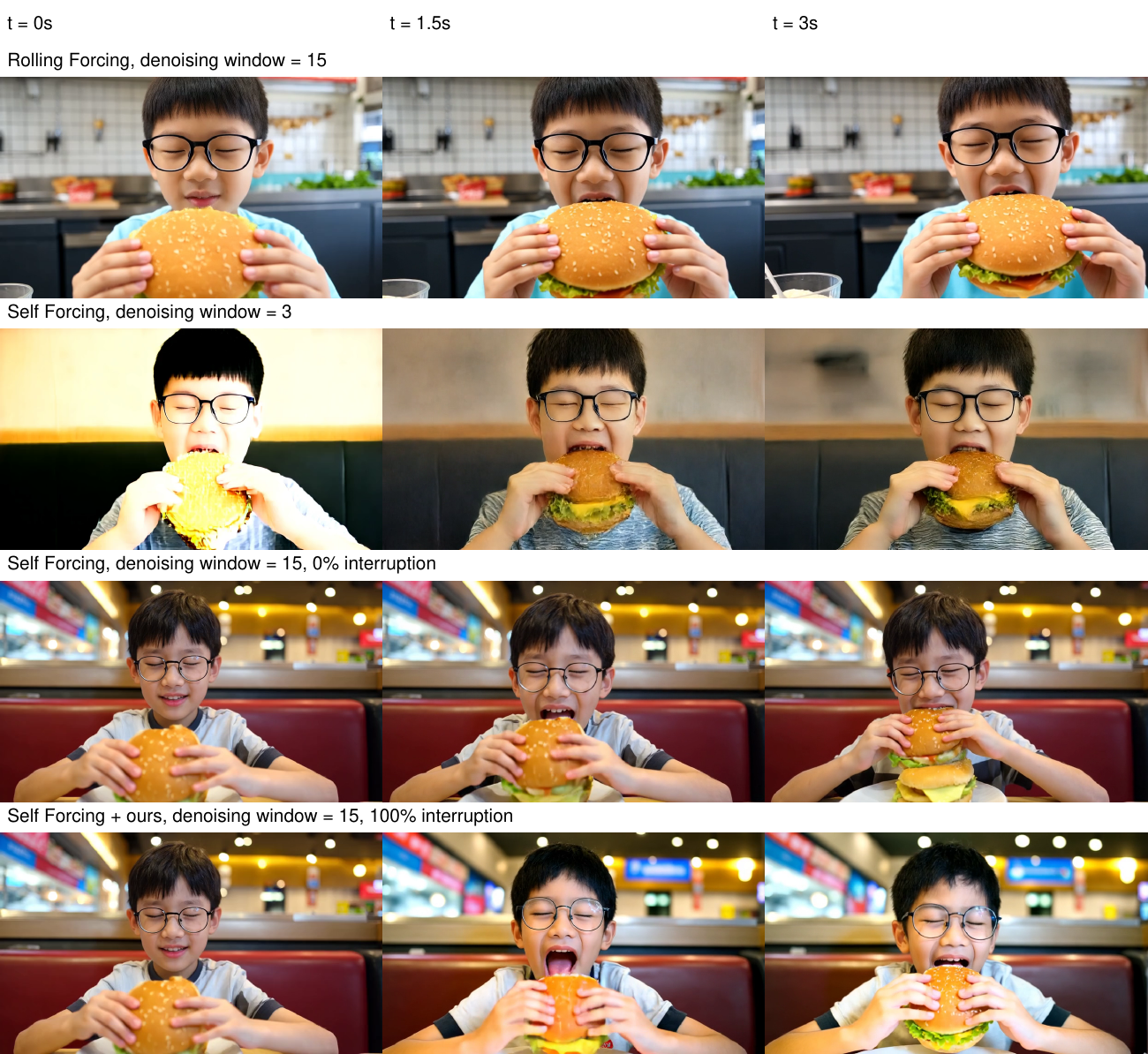}
  \captionof{figure}{\textbf{Eating a cheeseburger.} Large-chunk sampling retains facial and scene detail while allowing hand and facial motion under frame-wise commitment.}
  \label{fig:rf-interruption-burger}
\end{minipage}
\vspace{1em}
}]

\noindent\textbf{Prompt.} A Chinese boy wearing glasses enjoys a delicious cheeseburger with his eyes closed in a fast food restaurant.

\twocolumn[{
\begin{minipage}{\textwidth}
  \centering
  \includegraphics[width=\linewidth]{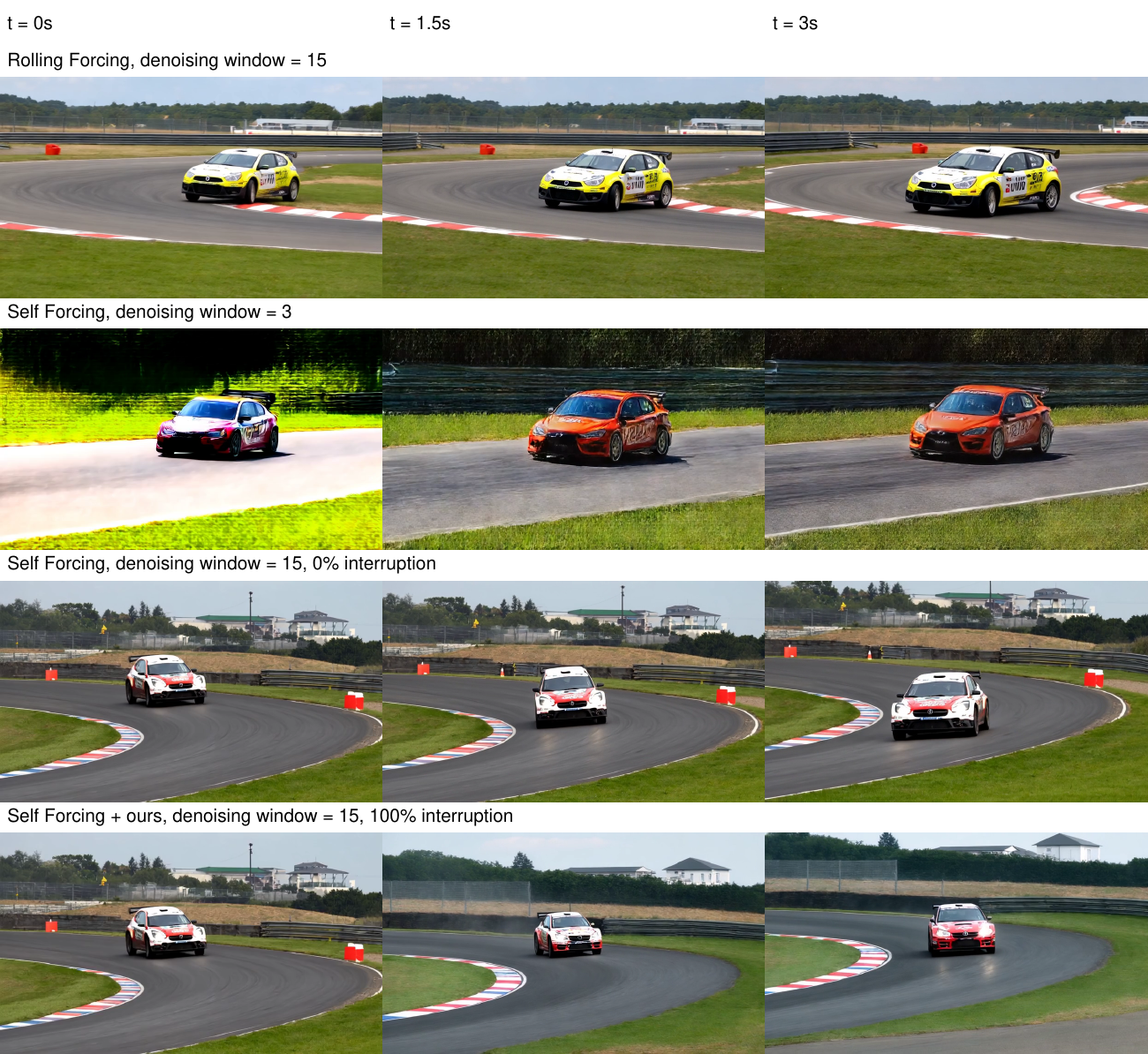}
  \captionof{figure}{\textbf{Rally car turning.} Frame-wise commitment produces greater car displacement while retaining the overall track layout.}
  \label{fig:rf-interruption-rally}
\end{minipage}
\vspace{1em}
}]

\noindent\textbf{Prompt.} A rally car taking a fast turn on a track.

\paragraph{Implementation and score aggregation.}
All training runs continue from their respective iteration-1,100 generator, critic, and EMA weights with freshly initialized optimizers, using the same continuation procedure to iteration 3,000. Each evaluation generates 135 latents and retains the first 480 decoded frames. For interrupted generation, clean KV features are computed from the complete fifteen-latent chunk before retaining the committed prefix. Interruption trajectories share random draws across rates, and fresh noise is used after replacement. The average quality score is the normalized weighted mean of the seven VBench quality dimensions: dynamic degree has weight 0.5 and each other dimension has weight 1. Imaging drift is reported separately as the mean per-video absolute difference between the first and last five seconds.

\clearpage
\flushbottom
\subsection{Experimental implementation details}
\label{app:experiment-details}

\paragraph{Quality evaluation and sampling.}
LTX uses three DiT sampling steps with schedule $[999,757,522,0]$, where
0 is the terminal boundary. Its chunk-size quality evaluation uses the
Internal32 set, 361-frame outputs at $832\times448$ and 24 fps, and matched
seeds across configurations. The interruption experiment uses
30 image--prompt pairs and a fixed prompt throughout each video. Its interruption
events trigger prediction replacement without changing the prompt text.

\paragraph{Forward profiling.}
Wan profiles use batch one and BF16, shared frame-wise/four-step EMA weights,
and synthetic inputs with a 21-latent prefix plus $K$ current latents.
One reference checkpoint is shared across nominal step configurations.
We report the slowest-rank median over 30 synchronized forwards after
10 warm-ups. Timings include forward communication and output gathering,
but exclude text encoding, VAE decoding, and the sampler loop. These
profiles measure the DiT forward component of generation.
LTX timings are cached-forward measurements for each $(K,\mathrm{USP})$ pair.

\paragraph{Serving configurations.}
Wan uses $K=3$ with USP4 on four B300 DiT GPUs and $K=21$ with USP8 on eight,
with its real VAE pipelined on a separate GPU, at $832\times480$ and 24 fps.
LTX uses $K=3$ with USP1 on one DiT GPU and $K=15$ with USP2 on two;
both have separate text and media GPUs, at $832\times448$ and 24 fps.
Small-chunk LTX is evaluated in both eager and CUDA Graph execution modes;
the large-chunk implementation uses eager execution and does not support
CUDA Graph capture. Small-chunk eager execution under-runs playback, motivating
the additional CUDA Graph configuration. CUDA Graph capture and replay
apply to the DiT. Sol-H3 uses 22-frame proposals with USP4 on four B300 GPUs
and 124-frame proposals with USP8 on eight, at $832\times480$ and 24 fps.
For each small/large configuration, baseline and our method share model
placement, continuation, decoding, playback, and scheduling lead time.

\paragraph{Response and utilization measurements.}
Response latency is measured from condition arrival to actual display after
minimum-RTT clock alignment. The median and 95th percentile correspond to
$L_{50}$ and $L_{95}$. GPU active time is time-weighted NVML GPU kernel activity,
averaged over the monitored GPUs; playback stall is the fraction of the
60-second session spent in playback underflow. Wan monitors its DiT GPUs,
excluding the separate VAE GPU. LTX averages over all three/four service GPUs,
including its dedicated text and media GPUs. Activity percentages use these model-specific monitoring scopes.

\paragraph{Comparison with predicted waiting reductions.}
For Sol-H3, the measured mean latency reduction is 0.493\,s for small chunks
and 2.814\,s for large chunks, compared with theoretical reductions of
0.438\,s and 2.562\,s, respectively. Wan's reductions increase from
0.155\,s for small chunks to 1.744\,s for large chunks; LTX eager reductions
increase from 0.014\,s to 2.116\,s. These results follow the predicted trend
that larger chunks leave more boundary waiting to remove.

\section{Conditioning and output adapters}
\label{app:adaptation}

The controller uses model-specific adapters to continue from committed content.
The conditioning state $\mathcal H_j$ contains prefix KV, a boundary image,
or retained latents, depending on the generator.

\paragraph{Prefix KV.}
For history-conditioned DiT sampling, clean prefill processes the complete
sampled chunk with within-chunk bidirectional attention. At the selected
boundary, the adapter keeps the committed prefix's KV entries from this
full-chunk pass and discards the uncommitted suffix's entries, preserving
positions and previously committed history.

\paragraph{Boundary-image continuation.}
The bidirectional image-conditioned path retrieves the full-resolution
RGB frame immediately preceding the selected boundary from the retained
timeline, then submits it through the checkpoint's native image
conditioning interface with the latest available prompt.
The H3 image-conditioned path can implement this operation without a new
temporal causal mask or wrapper-specific training.
The boundary image supplies the visual context for continuation.

\paragraph{Latent and multi-frame variants.}
The H3 experimental latent-input path replaces image encoding with a
normalized clean video latent while retaining condition packing and
condition-noise handling. This differs from copying the latent into the target
noise tensor, and its conditioning distribution differs from that of a single
encoded image. A separate multi-frame path supplies a
phase-aligned group of retained latents as keyframe conditions and removes
the corresponding output overlap. These variants change context, eligible
boundaries, and preparation costs, and are reported separately from the
boundary-image adapter.

\paragraph{Output mapping and VAE state.}
A native clip may include a regenerated anchor or context span. The output
adapter omits that span and maps only the remaining pixels to the new
playback interval. It retains source indices for subsequent anchors, and
never selects history from a discarded or trimmed-away prediction.
This mapping can have nonuniform latent-to-frame spans and fewer playable
groups than sampled latents; the controller uses their actual deadlines.

Streaming VAE decoding can publish finalized groups before the entire
clip has finished decoding, while DiT sampling remains joint over the
chunk. VAE overlap and any state dependent on a replaced suffix
must be handled before pixels are committed; the adapter may instead
decode the full clip and deliver its completed frames incrementally.
Group-ready and job-finish events remain separate in both cases.
History encoding, cache construction, VAE decoding, and mandatory overlap
work are included in the corresponding timing and cost measurements.

\section{Theoretical and Formal Details}
\label{app:theory}

\subsection{Timing assumptions and derivations}
\label{app:latency}

\subsubsection{Fixed chunk boundaries}

The reference model has playback boundaries at integer multiples of $T$,
fixed generation time $C$, and a condition available at time $t_{\mathrm{arrival}}$.
The new condition is unknown before $t_{\mathrm{arrival}}$, and a generation job fixes its condition at
launch. We exclude in-flight condition injection or partial DiT sampling reuse
that could shorten the new condition's $C$-second generation. We consider an isolated update that is not
subsequently superseded, with no added transport or client-buffer delay.
For attainability of a steady single-worker schedule, take $C\leq T$ and
assume the required prefix is available when that schedule starts a job.

\begin{proposition}
For uniform condition-arrival phase modulo $T$, the fixed-boundary model
has $\E[L]\geq C+T/2$. This lower bound is attainable under the stated
ideal scheduling assumptions.
\end{proposition}

\begin{proof}
The new result cannot finish earlier than $t_{\mathrm{arrival}}+C$. Its earliest allowed
playback boundary is therefore
$t_{\mathrm{playback}}^{\min}=T\lceil(t_{\mathrm{arrival}}+C)/T\rceil$.
The phase $U=(t_{\mathrm{arrival}}+C)\bmod T$ is uniform on $[0,T)$ because $C$ is fixed.
Ignoring the zero-measure case of exact alignment,
$t_{\mathrm{playback}}^{\min}-(t_{\mathrm{arrival}}+C)=T-U$,
whose expectation is $T/2$.
Thus $\E[t_{\mathrm{playback}}-t_{\mathrm{arrival}}]\geq C+T/2$.
For attainability, start generation at $nT-C$ and finish at $nT$.
An update is used by the next generation start after its arrival.
Those starts are separated by $T$, so their mean forward waiting time
under a uniform arrival phase is $T/2$. Generation takes $C$ and its
completion is aligned with playback.
\end{proof}

For any realized start $t_{\mathrm{start}}$ and display $t_{\mathrm{playback}}$,
the two waiting terms obey
\begin{equation}
 \begin{aligned}
 &(t_{\mathrm{start}}-t_{\mathrm{arrival}})
   +[t_{\mathrm{playback}}-(t_{\mathrm{start}}+C)]\\
 &\qquad=t_{\mathrm{playback}}-t_{\mathrm{arrival}}-C.
 \end{aligned}
\end{equation}
Changing $t_{\mathrm{start}}$ alone redistributes waiting without changing the total for a
fixed $t_{\mathrm{playback}}$. The $T/2$ term accounts for this combined waiting in the ideal baseline.
Additional prefix or resource dependencies contribute further waiting.

\subsubsection{Latent-boundary alignment}

If the prefix is immediately available, there are no extra buffer constraints,
and generation takes fixed time $C$, the earliest latent-aligned display is
$t_{\mathrm{playback}}=\delta\lceil(t_{\mathrm{arrival}}+C)/\delta\rceil$.
Uniform phase modulo $\delta$ gives
$\E[L]=C+\delta/2$ in this ideal case. With nonzero $R$, however, the
relevant phase is $(t_{\mathrm{arrival}}+R+C)\bmod\delta$.
The distribution of $Q$ therefore depends on the joint distribution of
$t_{\mathrm{arrival}}$ and $R$.

\subsubsection{Residual generation under an isolated update}

Assume the no-update system runs one indivisible job of length $C$ in each
period $T$, with $C\leq T$. A uniformly arriving isolated update lands
inside generation with probability $C/T$. Suppose that job must complete
to supply the prefix required for the new generation. Conditional on arriving
during it, the residual duration is uniform on $[0,C]$. When generation
is idle, the prefix and resources needed for the new generation are assumed
ready. Then
\begin{equation}
  \E[R]
  =\Pr(\mathrm{busy})\E[R\mid\mathrm{busy}]
  =\frac{C}{T}\frac{C}{2}
  =\frac{C^2}{2T}.
\end{equation}
Adding the new generation time and actual alignment gives
$\E[L]=C+C^2/(2T)+\E[Q]$.
For the illustrative values $T=5$\,s and $C=1.6$\,s, the residual mean is
$0.256$\,s. Excluding $Q$, an idle arrival responds in $1.6$\,s, a busy
arrival averages $2.4$\,s, and the unconditional average is $1.856$\,s.

This expression assumes equal blocking and new-generation durations and
an isolated update. For a pipelined implementation or repeated updates,
the same calculation uses the measured busy fraction and residual
critical-path duration.

An instructive limiting case is $C=T=K\delta$. The generation worker is always busy
in the reference schedule, so $\E[R]=T/2$, and its new completion is
latent-aligned with $Q=0$. The resulting mean is $3T/2$, equal to the
ideal fixed-chunk baseline. In this fully occupied schedule, prefix waiting offsets the benefit of
finer delivery boundaries.

\subsection{Prediction-use accounting}
\label{app:cost}

\subsubsection{DiT sampling-call accounting}

Let $N$ equal-cost complete DiT sampling calls each produce $K$ latents, and let $U$
of those latents ultimately be played. If $w=1-U/(NK)$, the number of
DiT sampling calls per played latent, relative to full-use chunk generation,
is
\begin{equation}
  A=\frac{N/U}{1/K}=\frac{NK}{U}=\frac{1}{1-w}.
\end{equation}
Here $A$ measures the increase in complete DiT sampling calls per played
latent. End-to-end cost also includes conditioning preparation, cache
construction, VAE decoding, and other pipeline work. Initialization and
unused final buffers use the same accounting across methods.

For adapters that trim overlapping context, normalize by their full-use
output yield and account for mandatory overlap separately from update-driven
discarding. When output groups have different durations, count delivered
frames or seconds.

\subsubsection{Independent-update model}

For intuition, consider a discrete model in which every $\delta$-period
decision independently triggers an update with probability $p$.
Each DiT sampling call provides $K$ completed predictions, its head is always used,
and generation is assumed ready at the necessary decision without
cross-period prefix blocking. Let $M$ be the number used before the next
update or exhaustion. For $m=1,\ldots,K$,
$\Pr(M\geq m)=(1-p)^{m-1}$. The tail-sum identity yields
\begin{equation}
  \begin{aligned}
  \E[M]&=\sum_{m=1}^{K}(1-p)^{m-1}
  =\frac{1-(1-p)^K}{p},\\
  A&=\frac{K}{\E[M]}
  =\frac{Kp}{1-(1-p)^K}.
  \end{aligned}
\end{equation}
The continuous extension at $p=0$ is $A=1$, and $p=1$ gives $A=K$.
If Poisson arrivals of rate $\lambda$ are coalesced within each interval,
$p=1-\exp(-\lambda\delta)$. For $\lambda T\ll1$,
\begin{equation}
  A-1 \simeq \frac{\lambda\delta(K-1)}{2}
  =\frac{\lambda(T-\delta)}{2}.
\end{equation}

In the online controller, finite generation time and coalesced updates
change prediction reuse. Its sampling overhead can be estimated using the
measured distribution of $M$. If a sampling call occupies its worker for
$C_{\mathrm{DiT}}$, the average load is $\rho=AC_{\mathrm{DiT}}/T$;
$\rho<1$ leaves average capacity for this workload. Meeting individual
playback deadlines also requires the boundary checks below.

\subsection{Extended finite-state controller specification}
\label{app:algorithm}

This specification formalizes the control logic of Section~\ref{sec:fsm} in the main paper as
an extended FSM whose two control states describe worker activity and whose
registers carry the remaining system state. The finite control set is
$\mathcal S=\{\mathsf{Produce},\mathsf{Consume}\}$, with the following
invariant after each control-clock tick:
\begin{equation}
 z=\mathsf{Produce}\ \Longleftrightarrow\ J\neq\varnothing.
 \label{eq:fsm-state}
\end{equation}

\paragraph{Registers and events.}
The register tuple is
$x=(B,\mathcal H,V,v,u,J,p,\ell,n,\beta)$.
$B$ stores versioned predictions and their latent/pixel readiness;
$\mathcal H$ stores conditioning inputs and model-specific state, including prefix
KV when used; $V$ stores VAE state and the source-to-playback mapping.
$v$ is the latest accepted
condition version and $u$ a pending version or $\varnothing$.
An active $J$ records a unique job identifier, fixed condition version,
boundary, and readiness/deadline estimates.
$p$ is the delivery cursor, advancing as outputs are handed to the RTC
stream, and $n$ is the end of valid predictions.
$\ell$ is the exclusive end of the \emph{committed prefix}: latents fixed
against further modification, whether or not they have been delivered.
Launching a job also commits the prefix on which it depends. An irrevocable
delivery advances it to $\max(\ell,p)$, and $\operatorname{Launch}$ at $j$
sets it to $j$; since every feasible boundary satisfies $j\geq\ell$,
$\ell$ never decreases. The playback plan $\beta$ assigns deadline $b_j$ to position
$j$. A monotone event sequence records the processing order, including
events sharing a timestamp.

The normal-operation event alphabet is
\begin{equation}
 \mathcal E=\{\mathsf{update},\mathsf{ready},\mathsf{finish},
              \mathsf{delivery}\}.
\end{equation}
Events carry condition versions, job identifiers, positions, and timestamps
as appropriate. Only valid callbacks from the corresponding job may publish
results; stale callbacks cannot reinstall an invalidated suffix.
Events are recorded as they arrive; control transitions and dispatch are
evaluated once per control-clock tick after processing the recorded events.
The control clock is separate from the scheduled delivery cadence.

\paragraph{Feasible boundaries and dispatch guard.}
For a candidate boundary $j$, define
\begin{equation}
 \begin{split}
 F_j(x,t)={}&[j\geq\ell]\ \land\operatorname{ContextReady}(x,j)
 \land\operatorname{Cover}(x,t,b_j)\\
 &{}\land\operatorname{ValidVAE}(x,j)
 \land\bigwedge_{m=0}^{M_j-1}
       [t+\widehat C_{j,m}+\epsilon\leq b_{j+m}].
 \end{split}
 \label{eq:fsm-feasible}
\end{equation}
$\operatorname{ContextReady}$ requires the inputs needed by the selected
adapter $\Phi_\theta$. For prefix KV continuation, these include
the prefix latents and their KV.
For boundary-image continuation, the retained boundary RGB frame suffices
as the native image encoder's input.
$\operatorname{Cover}$ requires that the old prefix can meet all remaining
playback deadlines before $b_j$; $\operatorname{ValidVAE}$ requires a
continuation that preserves already committed pixels.
$M_j$ is the number of newly playable groups after the adapter's overlap
trimming and must be positive for dispatch; $M_j=K$ for the untrimmed
uniform-latent model.
A group contains the decoded frames assigned to one delivery unit;
one VAE decoding window may make several groups ready together.
$\widehat C_{j,m}$ estimates time from launch until new group $m$ becomes
deliverable, and $\epsilon$ is a configured safety margin. Deadlines
$b_{j+m}$ come from the actual output mapping, including nonuniform native
group durations where needed.
The feasible set $\mathcal F(x,t)$ contains positions with $F_j(x,t)$ true.
For a pending update, it includes an append boundary when that boundary
is feasible.

Let $h(x,t)$ be remaining continuous playable coverage and
$\eta(x,t)>0$ the configured refill lead, including alignment and margin.
Define $\operatorname{Refill}(x,t)=[h(x,t)\leq\eta(x,t)]$ and
\begin{equation}
 \begin{split}
 g(x,t)=[J=\varnothing]\ \land\bigl(
 &[u\neq\varnothing\land\mathcal F(x,t)\neq\varnothing]\\
 {}\lor{}&[u=\varnothing\land\operatorname{Refill}(x,t)\land F_n(x,t)]
 \bigr).
 \end{split}
 \label{eq:fsm-guard}
\end{equation}
With a pending update, choose the earliest feasible $j$ by playback deadline;
otherwise append at $n$. Generation is launched for either refill or a pending update once its
boundary is feasible.

\paragraph{Event actions and transition semantics.}
At each control-clock tick $t$, apply the register actions $U_e$ for all
events recorded since the previous tick, in arrival order, obtaining
$\bar x$. Then evaluate the dispatch guard $g(\bar x,t)$ once, including
time-dependent refill and retry checks even when no new event arrived.
Let $z$ be the control state and
$\bar J$ the active-job entry in $\bar x$. The transition is
\begin{equation}
 (z^+,x^+)=
 \begin{cases}
  (\mathsf{Produce},\operatorname{Launch}(\bar x)),&g(\bar x,t),\\
  (\mathsf{Produce},\bar x),&\bar J\neq\varnothing,\\
  (\mathsf{Consume},\bar x),&\text{otherwise}.
 \end{cases}
 \label{eq:fsm}
\end{equation}
Dispatch requires $\bar J=\varnothing$, so the first two guards are disjoint.
Only $\operatorname{Launch}$ starts a DiT sampling call. This order ensures that a
completion's outputs are validated before dependent generation is launched.
\begin{enumerate}
  \item $\mathsf{update}$ sets $v$ and $u$ to the newly accepted version,
  recording any replaced pending version as superseded. It leaves $J$ and
  the committed prefix unchanged. The next control-clock tick evaluates the pending update.
  \item $\mathsf{ready}$ validates and publishes the indicated latent,
  KV, or pixel readiness in $B,\mathcal H,V$ without releasing $J$.
  Pixel delivery may begin at its scheduled boundary during
  $\mathsf{Produce}$ once the corresponding pixels are ready.
  \item $\mathsf{finish}$ validates remaining results and releases the
  active scheduling slot by clearing $J$. At the control-clock tick,
  the state remains $\mathsf{Produce}$ if $g$ authorizes another job,
  and becomes $\mathsf{Consume}$ otherwise.
  Pending versions survive completion until actually dispatched.
  \item $\mathsf{delivery}$ delivers the next valid, ready group and advances
  $p$ in either control state. Missing output records an underrun. Delivery advances $\ell$ when downstream
  transfer is irrevocable. It invokes no DiT sampling itself;
  refill or update demand is checked by the controller loop.
\end{enumerate}

$\operatorname{Launch}$ atomically commits the prefix before $j$,
invalidates only the uncommitted suffix and dependent speculative state,
and records a new $J$ for a complete $K$-latent DiT sampling call with
the latest condition. It clears only the pending version it consumes.
VAE decoding supplies the required pixels for each group's deadline.
In-flight conditioning is fixed; later arrivals update the pending version $u$.
Initialization uses a separate bootstrap launch before the playback clock
starts, establishing initial coverage and VAE state. Startup latency is
recorded separately from steady-state response.

\paragraph{Required invariants.}
The committed prefix is never overwritten. A new generation never directly
attends to invalid suffix entries. At most one job occupies the active slot,
and only a guarded launch starts new DiT sampling.
Playback and cache positions remain consistent after replacement. No predicted latent is counted as displayed before a display event.
VAE decoder invalidation cannot alter already delivered pixel frames. Missed
deadlines, superseded updates, and additional work are recorded in the logs.

\paragraph{Runtime estimates and failure handling.}
The implementation estimates $C$ at runtime. The experimental configuration
specifies the estimator, safety margin, retry policy, and any fallback that
freezes playback, drops frames, uses stale content, or alters the target
boundary. Quality and response measurements include these fallback outcomes.

Job failure publishes no invalid outputs and releases the scheduling slot
only when the worker confirms termination. The declared fallback handles any
resulting coverage loss.
Delivery continues in both states; the controller checks remaining work
and dispatch demand on each control-clock tick.

\subsection{Event and measurement schema}
\label{app:logging}

Generation produces deliverable video. For the latent diffusion models
studied here, $C$ is the critical-path time from launch to the first
deliverable output group, including conditioning preparation, DiT sampling,
and the required VAE decoding. This definition accounts for overlap between
stages. $R$ covers scheduling, resource, and conditioning-input readiness
waits; $Q$ covers transport, buffering, and playback alignment. The serving
experiments measure response from condition arrival to actual display of
output generated under the updated condition.

The relation $T=Kr/f$ uses a uniform latent-to-frame mapping. Models with
nonuniform mappings or overlap trimming use their actual output frame counts;
$f$ is the video playback frame rate.

The serving evaluation uses the same 30 one-minute interaction sequences
across configurations, with condition updates arriving independently of
playback boundaries. Within each configuration, the baseline and our method
share hardware, model placement, decoding, and scheduling lead time.
Appendix~\ref{app:experiment-details} specifies the deployment configurations
and response, GPU activity, and playback-stall measurements.

\section{Future work: co-optimizing DiT sampling and VAE decoding}
\label{app:future}

Accelerating joint DiT sampling is one route to retaining large chunks,
but the serving objective is to reduce generation time $C$ to the
first deliverable output group. A natural next step is to co-design DiT
sampling and VAE decoding as a single generation pipeline. This includes when completed latents
are handed to the VAE decoder, how decoding is scheduled against playback
deadlines, and how compute and memory are allocated between stages.

Overlap schedules must respect the data dependencies of temporally
bidirectional DiT sampling and VAE decoding. Candidate schedules should measure
time to the first deliverable group, later-group deadlines, total GPU cost,
and playback stalls. Parallel acceleration should be evaluated within the
complete pipeline and resource budget.

\end{document}